\documentclass{article}
\usepackage{graphicx}
\usepackage{amssymb}
\usepackage{amsthm}
\usepackage{color}
\usepackage{amsmath}
\usepackage{fullpage}
\usepackage{setspace}
\usepackage{charter}%I like this font...

\usepackage{hyperref} 
\hypersetup{backref, colorlinks=true, citecolor=blue, linkcolor=blue}
\newcommand{\rref}[1]{\hyperref[#1]{\ref*{#1}}}

\newcommand{\m}[1]{\mathrm{#1}}

\newtheorem{theorem}{Theorem}[section]
\newtheorem{lemma}[theorem]{Lemma}

\newtheorem{corollary}[theorem]{Corollary}

\theoremstyle{definition}
\newtheorem{definition}[theorem]{Definition}

\newcommand{\ale}[1]{\textcolor{red}{{\bf ALE $\star$ #1 $\star$}}}

\usepackage[round, authoryear]{natbib}%[square]
\title{On the Geometry of Discrete Exponential Families %with polyhedral support 
with Application to Exponential Random Graph  Models
\author{Stephen E. Fienberg\thanks{Email: {\tt fienberg@stat.cmu.edu}}\\
Department of Statistics, Machine\\ Learning Department and Cylab\\
Carnegie Mellon University\\
Pittsburgh, PA 15213-3890 USA
\and
Alessandro Rinaldo\thanks{Email: {\tt arinaldo@stat.cmu.edu}}\\ 
Department of Statistics\\
Carnegie Mellon University\\
Pittsburgh, PA 15213-3890 US
\and
Yi Zhou\thanks{Email: {\tt yizhou@stat.cmu.edu}}\\ 
Machine Learning Department\\
Carnegie Mellon University\\
Pittsburgh, PA 15213-3890 USA}
\date{}
}

\begin{document}
\maketitle

\begin{abstract}
There has been an explosion of interest in statistical models for analyzing network data, and considerable interest in the class of exponential random graph (ERG) models, especially in connection with difficulties in computing maximum likelihood estimates.  The issues associated with these difficulties relate to the broader structure of discrete exponential families.  This paper re-examines the issues in two parts.   First we consider the closure of $k$-dimensional exponential families of distribution with discrete base measure and polyhedral convex support $\mathrm{P}$. We show that the normal fan of $\mathrm{P}$ is a geometric object that plays a fundamental role in deriving the  statistical and geometric properties of the corresponding extended exponential families. We discuss its relevance to maximum likelihood estimation, both from a theoretical and computational standpoint.
Second, we apply our results to the analysis of ERG models. In particular, by means of a detailed example, we provide some characterization of the properties of ERG models, and, in particular, of  certain behaviors of ERG models known as degeneracy. 
\end{abstract}

\section{Introduction}

Our motivation for the work described in this paper comes from the analysis of network data using models representable by graphs, where the nodes correspond to individuals and the edges to relations or  linkages among them.  
Such graphical representation has a long history, dating back to \cite{moreno:1934}, and was recast within the exponential family framework by \cite{HL:81} and \cite{FS:86} \citep[see also][]{SI:90}.
%Such a graphical representation has its origins in work by~\cite{moreno:1934} in the 1930s and a class of statistical models now widely used for the analysis of network data are an outgrowth of work in the 1970s, especially the $p_1$ model proposed by~\cite{HL:81}, that made explicit use of exponential family structure. 
Their work led to the development of the broader class of exponential random graph (ERG), or $p^*$, models for social networks \citep[see, e.g.][]{WP:96}, but likelihood methods for their analysis remained out of reach until earlier this decade. For a broad review of these and other network models, see~\cite{GZFA:2009}.  Recent work on maximum likelihood estimation for ERG models, however,  has pointed to difficulties that have been characterized as ``degeneracies" or ``near degeneracies" by \cite{H:03} and \cite{HGH:08}.  The explanation for these difficulties lies within broader characterizations of ``degeneracies" for discrete exponential families.

Exponential families are one of the most important and widespread class of parametric statistical models, whose remarkable properties have long been established in the statistical literature \citep[see, e.g.,][]{BRN:78,BRW:86,LETAC:92}. Among the most interesting features of exponential families is the notion of the closure  of the family, known as the extended exponential family, whose mathematical theory has been recently worked out in great generality \citep[see][]{CS:01,CS:03,CS:05,CS:08}. 
 The study of the extended families  is particularly important, as it may directly pertain to the existence of the maximum likelihood estimates and to the estimability of the natural parameters. 
This is the case for discrete exponential families, for which the maximum likelihood estimates may not exist with some positive probability. A notable instance is the class of log-linear models, for which existence of the MLE and closure of the family can be characterized in a purely geometric fashion \citep[see, e.g.,][]{ERIKSSON:06,GMS:06,ALE:06a}.
%\ale{Motivation. Summary of the paper.}

In this article we are concerned with discrete linear exponential families. In the first part of the paper, we show that the geometric and statistical properties of the extended family depend in a fundamental way on the normal fan of the convex support. In particular, the normal fan can be used to characterize non-identifiability of the families in the closure, to represent the densities in the extended family as almost sure limits of the densities in the original family along certain directions of the parameter space and to describe the directions of recession of the (negative) log-likelihood function. 

As an application of our results, in the second part of the paper we turn our attention to exponential random graph models, a particular class of discrete linear exponential families. %These models describe probability distributions over the set of all graphs with a fixed number of nodes and are extensively used for modeling social networks: see, for example, \cite{HL:81}, \cite{FS:86}, \cite{WP:96} and, for recent reviews of ERG models with applications to social networks, \citet{WR:04}, \citet{R:07}, \citet{RA:07} and references therein.
Our discussion is based an the detailed analysis of the ERG model on a the graphs on 9 nodes with two-dimensional sufficient statistics consisting of the number of edges and the number of triangles. We use Shannon's entropy function to illustrate graphically how concentrated the distributions in this family are, viewed as functions of both the natural and mean value parameters. Besides illustrating the theoretical results derived in the first part of the article, our analysis sheds  light on a variety of pathological behaviors observed in practice while fitting ERG models known as degeneracy \citep[see, e.g.,][]{H:03}, and, more generally, on the qualities and attributes of ERG models. Our analyses indicate that  perhaps network analysts and methodologists attribute to ERG models a degree of regularity that they may not possess. 

The remainder of this article is organized as follows. In Section \ref{sec:polysupp} we provide the derivation of our ket theoretical results. In Section \ref{sec:setting}, we begin by describing our settings and briefly review the theory of extended exponential families and their fundamental properties.  Then Section \ref{sec:extfan}, we introduce the notions of normal cones and the normal fan to the convex support of the family. In Section \ref{sec:mainresult} we state our main result and a discussion of its corollaries, while Section \ref{sec:comput} presents come computational considerations concerning maximum likelihood estimation for extended exponential families. Section \ref{sec:ERG} consists of an application of our results to ERG models. First  in Section \ref{sec:ERGintro} 
we introduce the class of ERG models and  then in Section \ref{sec:one} 
 we present our running example of an ERG model on the set of all graphs on 9 nodes. We next introduce the concept of degeneracy for ERG models  in Section \ref{sec:degeneracy}, while in Section \ref{sec:entropy} we use our theoretical results  to illustrate graphically the features of the model in the running example of Section \ref{sec:one} to show how degeneracy arises. 
The appendices contains the proofs and some additional result on how to establish existence of the maximum likelihood estimates in discrete linear exponential families using linear programming.

We end this section by establishing the notation we will be using throughout. For two vectors $x$ and $y$ in $\mathbb{R}^d$, $\langle x, y \rangle = \sum_{i=1}^d x_i y_i$ denotes their inner product. The Eucludean norm of a vector $x$ is $\| x \|_2 = \sqrt{\langle x, y \rangle}$. If $A$ is a  subset of $\mathbb{R}^d$, we indicate with $\m{convhull}(A)$ its convex hull and with $\m{cone}(A)$ the set of all of its conic combinations. Finally, for any $A \subset \mathbb{R}^d$, possibly of dimension less than $d$, its relative interior $\m{ri}(A)$ is its interior relative to $\m{convhull}(A)$.

\section{Extended Exponential Families with Polyhedral Support}\label{sec:polysupp}

\subsection{Settings}\label{sec:setting}

In this section we introduce the statistical and geometric background needed for our results. 
%Although we briefly review the notions of extended exponential families of distributions and of the extended MLE,
We will assume throughout some familiarity with the general theory of exponential families and the basics of polyhedral geometry. For more complete treatments, consult \cite{BRN:78}, \cite{BRW:86}, \cite{CS:01,CS:03,CS:05,CS:08} and \cite{ALE:06a} for material on exponential families,  and  \cite{ZIE:95} and \cite{SCH:98} for introductions to polyhedral geometry.

We consider an exponential family of distributions $\mathcal{E}_{\mathrm{P}}$  on $\mathbb{R}^k$ with densities
\[
p_{\theta}(x) = \exp \left\{ \langle x, \theta \rangle - \psi(\theta) \right\}, \quad \theta \in \Theta,
\]
with respect to some base measure $\nu$, where
\[
\Theta \subseteq \{  \theta \in \mathbb{R}^k \colon \int_{\mathbb{R}^k} \exp^{\langle x, \theta \rangle } d \nu(x) < \infty\}
\]
is the natural parameter space and $\psi(\theta) = \log \int_{\mathbb{R}^k} \exp^{\langle x, \theta \rangle } d \nu(x)$ the log-partition function. 
The support of $\mathcal{E}_{\m{P}}$ is the closure of the set $\{ x \colon \nu(x) > 0\}$, while the \textit{convex support} $\mathrm{P}$ is the closure of the convex hull of the support of $\mathcal{E}_{\mathrm{P}}$. 
We will assume throughout the paper that 
\begin{itemize}
\item[(A1)] $\nu$ has countable support;
\item[(A2)] $\mathrm{P}$ is a full-dimensional polyhedron in $\mathbb{R}^k$, that is, $\m{P}$ does not belong to any proper  affine subspace of $\mathbb{R}^k$;
\item[(A3)] for each face $F$ of $\m{P}$, $F = \m{convhull}(S_F)$, for some set $S_F \subset \m{supp}(\nu)$;
\item[(A4)] the natural parameter space $\Theta$ is an open set.
\end{itemize}  
Assumptions (A1) and (A2) imply, in particular, that the family is in minimal form and, therefore, identifiable. 
We remark that assumption (A2) is not necessary and is imposed to simplify the exposition; our results would still hold  with some minor changes without assumption (A2), and the cost of additional technicalities in the proofs. In fact, any degenerate exponential family can be made full by taking appropriate affine transformations, a procedure known as reduction to minimality (see, e.g., Theorem 1.9 in \citealp{BRW:86} or Lemma 8.1 in  \citealp{BRN:78}).
Assumption (A3) is needed to guarantee the existence of probability distributions supported over the  boundary of $\m{P}$, which is an indispensable feature of the extended exponential family, described in the next section. It could be easily relaxed by allowing some faces to have zero $\nu$ measure. %, a  webut for ease of readability, we do not pursue this generalization in the paper. 
Finally, assumption (A4) is a standard. In particular, for our discussion of ERG models,  $\Theta = \mathbb{R}^k$.

%\ale{Mention $\mathcal{V}$ representation and the representation $\m{P} = \m{C} + \m{Q}$, where $\m{C}$ is a polyhedral cone and $\m{Q}$ a polytope.}

\subsubsection{Basics of Extended Exponential Families}
Letting $X = x$ be the observed sample from an unknown distribution in $\mathcal{E}_{\m{P}}$, the random set
\begin{equation}\label{eq:mle}
\widehat{\theta}(x) = \widehat{\theta} = \left\{  \theta^* \in \Theta \colon  p_{\theta^*}(x) = \sup_{\theta \in \Theta} p_{\theta}(x) \right\}
\end{equation}
is the \textit{maximum likelihood estimate}, or MLE, of $\theta$. If $\widehat{\theta} = \emptyset$ the MLE is said to be nonexistent. 
Existence of the MLE is determined by the geometry of $\m{P}$, as indicated by the following well-known, fundamental result (see, e.g., Theorem 5.5 in \citealp{BRW:86} or Proposition 4.2 \citealp{ALE:06a} for different proofs).

\noindent \begin{theorem}\label{thm:mle}
Under the current settings, the MLE $\widehat{\theta}$ exists and is unique if and only if $x \in \mathrm{relint}(\mathrm{P})$.
\end{theorem}

Furthermore, setting $\mathbb{E}_\theta(X) = \int_{\mathbb{R}^k} z p_{\theta}(z) d \nu(z)$, because of the minimality of $\mathcal{E}_{\mathcal{P}}$, the \textit{mean value parametrization} map
\[
\nabla \psi \colon \m{int}(\Theta) \mapsto \mathrm{relint}(\mathrm{P})
\]
given by
\begin{equation}\label{eq:homeo}
\nabla \psi(\theta) = \mathbb{E}_\theta(X),
\end{equation}
is a homeomorphism, so that one can equivalently represent any distribution in $\mathcal{E}_{\mathrm{P}}$ using the natural parameter $\theta$ or the mean value parameter $\mu = \mathbb{E}_\theta(X) \in \mathrm{relint}(\mathrm{P})$.
%defined for every $\mu = \mathbb{E}_\theta(X) \in \mathrm{relint}(\mathrm{P})$ and every $\theta \in \m{int}(\Theta)$, is a homeomorphism between $\mathrm{relint}(\mathrm{P})$ and $\m{int}(\Theta)$
In particular, if the MLE exists, it is determined by the equation
\[
\widehat{\theta} = \nabla \psi^{-1}(x),
\]
which translates into the moment equation $\mathbb{E}_{\widehat{\theta}}(X) = x$.

For any proper face $F$, let $\nu_F$ be the restriction of $\nu$ to $F$. Then, $\nu_F$ determines a new exponential family of distributions $\mathcal{E}_F$, with densities with respect to $\nu_F$ given by
\[
p^F_{\theta}(x) = \exp \left\{ \langle x, \theta \rangle - \psi^F(\theta) \right\}, \quad \theta \in \Theta_F,
\]
where the natural parameter space is $\Theta_F = \{  \theta \in \Theta \colon \int_{\mathbb{R}^k} \exp^{\langle x, \theta \rangle } d \nu_F(x) < \infty\}$ and the log-partition function is $\psi^F(\theta) = \log \int_{\mathbb{R}^k} \exp^{\langle x, \theta \rangle } d \nu_F(x)$. 
Notice that, since $\int_{\mathbb{R}^k} \exp^{\langle x, \theta \rangle } d \nu_F(x) \leq \int_{\mathbb{R}^k} \exp^{\langle x, \theta \rangle } d \nu(x)$, $\Theta = \Theta_F$.
By assumption (A3), the convex support of this new family is $F$ and the existence result of Theorem \ref{thm:mle} carries over: the MLE exists if and only if the observed sample $x$ belongs to $\mathrm{relint}(F)$. However, since $\mathcal{E}_F$ is supported on a lower-dimensional affine subspace of $\mathbb{R}^k$, it is no longer minimal, hence the MLE is not unique, and it consists instead of many solutions to (\ref{eq:mle}); see  Corollary \ref{cor:uno} below for details. Nonetheless, via reduction to minimality \citep[see, e.g.,][Theorem 1.9]{BRW:86}, it can be verified that, when $\widehat{\theta}$ is not empty, it consists exactly of those points satisfying the first order optimality conditions
\begin{equation}\label{eq:mle.F}
x = \nabla \psi_F(\theta), \quad \forall \theta \in \widehat{\theta},
\end{equation}
with the corresponding moment equations $\mathbb{E}^F_{\theta}(X) = \int_{\mathbb{R}^k} z p^F_{\theta}(z) d \nu^F(z) =  x$, $\forall \theta \in \widehat{\theta}$,  still holding. %, where $\mathbb{E}_\theta^F$ denotes expectation with respect to $\theta \in \Theta_F$.
In fact, lack of minimality bears not effect on the mean value parametrization: for every $\theta \in \Theta_F$, there exists one point $x \in \m{ri}(F)$ such that 
\begin{equation}\label{eq:Ex}
\mathbb{E}_{\theta}[X] = x,
\end{equation}
 and, similarly, for any $x \in \m{ri}(F)$, there exists a set $\theta_F \subset \Theta_F$, depending on $x$, such that  (\ref{eq:Ex}) holds for all $\theta \in \theta_F$. See equation (\ref{eq:thetaF}) below for a characterization of $\theta_F$.

The collection of distributions
\[
\mathcal{E} = \bigcup_F \mathcal{E}_F
\]
as $F$ ranges over all the faces of $\mathrm{P}$, including  $\mathrm{P}$ itself, is called the \textit{extended exponential family} of distribution. With respect to the extended family $\mathcal{E}$, for \textit{any} observed sample $X = x$, the MLE, or \textit{extended MLE}, is always well defined and is the set of solutions to (\ref{eq:mle.F}), where $F$ is the \textit{unique} face containing $x$ in its relative interior.

\subsection{Extended Exponential Families and The Normal Fan of $\m{P}$}\label{sec:extfan}
In this section we introduce the notion of normal fan of $\m{P}$ and establish its relevance for the extended family $\mathcal{E}$. See Lemma \ref{lem:nf} in Appendix B for some basic properties of the normal cones and of the normal fan. 

By assumption (A1) and (A2), there exists a $m \times k$ matrix  $A$ and a vector $b \in \mathbb{R}^m$ such that
\begin{equation}\label{eq:P}
\mathrm{P} = \{ x \in \mathbb{R}^k \colon A x \leq b \},
\end{equation}
where the system contains no implicit equalities. %, so that $\m{P}$ does not belong to a lower-dimensional affine subspace of $\mathbb{R}^k$.
%For $c \in \mathbb{R}^k$, the face $F_c$ of $\mathrm{P}$ is defined to be 
%\[
%F_c = \Big\{ x \in \mathrm{P} \colon \langle c, x \rangle  = \max_{y \in \mathrm{P}} \langle c, y \rangle \Big\}.
%\]
%Equivalently,
A proper face $F$ of $\mathrm{P}$ is a subset of $\mathrm{P}$ defined by
\begin{equation}\label{eq:F}
F =  \Big\{ x \in \mathrm{P}  \colon A_F x = b_F  \Big\},
\end{equation}
for some subsystem $A_Fx\leq b_F$ of $A x \leq b$ and, therefore, it is itself a polyhedron. The whole polyhedron $\mathrm{P}$ is regarded as the improper face of itself associated to the full system of inequalities, so that $\m{P}$ is representable as the disjoint union of the relative interiors of all its faces. The dimension  of a face $F$, $\m{dim}(F)$, is the dimension of the affine subspace it generates or, equivalently, the dimension of the null space of $A_F$. Faces of dimension $k-1$ are called facets of $\m{P}$ and, if the system (\ref{eq:P}) has no redundant inequality, something which can always be assumed without loss of generality, the number $m$ of rows of $A$ match the number of facets.
Equation  (\ref{eq:P}) is known the $\mathcal{H}$ representation of $\m{P}$. 
Alternatively, $\m{P}$ could be described using the $\mathcal{V}$ representation as the sum of a polytope and a polyhedral cone:
\begin{equation}\label{eq:Vrepre}
\m{P} = \m{Q} + \m{C},
\end{equation}
where the sign $+$ denotes Minkowski addition, and $\m{Q} = \m{convhull}(\mathcal{Q})$ and $\m{C} = \m{cone}(\mathcal{C})$, with $\mathcal{Q}$ and $\mathcal{C}$  two finite sets of vectors in $\mathbb{R}^k$.
Throughout the paper, we will rely on the $\mathcal{H}$ representation (\ref{eq:P}), which we find more suited to our purposes, although our results could be established using (\ref{eq:Vrepre}).

%We remark that assumption (A2) is imposed mainly for mathematical convenience, in order to simplify the exposition. In fact, any degenerate exponential family can be made full by reduction to minimality \citep[see, e.g.,][Theorem 1.9]{BRW:86}. Our results would still hold  with some minor changes without assumption (A2), and the cost of additional technicalities in the proofs.

For every face $F$ of $\mathrm{P}$, let 
\[
N_F = \Big\{ c \in  \mathbb{R}^k  \colon F \subseteq \{ x \in \mathrm{P} \colon \langle c, x \rangle  = \max_{y \in \mathrm{P}} \langle c, y \rangle  \}\Big\}
\]
be the polyhedral cone consisting of all the linear functionals on $\mathrm{P}$ that are maximal over $F$, called the \textit{normal cone} of $F$. Then, $\m{dim}(N_F) = k - \m{dim}(F)$, so that larger faces of $\m{P}$ correspond to smaller normal cones. 
By Lemma \ref{lem:nf} part {\it 5.}, the normal cone of a proper face $F$ can be equivalently defined as
\[
N_F = \mathrm{cone} \left( a_1, \ldots, a_{m_F} \right),
\]
where $a_i$ denotes the transpose of the $i$-th row of the submatrix $A_F$ given in (\ref{eq:F}), where $i =1\ldots,m_F$.

The collection of cones 
\[
\mathcal{N}(\m{P}) = \{ N_F, F \text{ is a face of } \mathrm{P} \}
\]
forms a polyhedral complex in $\mathbb{R}^k$ \citep[see, e.g.][]{STURM:95}, called the \textit{normal fan} of $\mathrm{P}$. Notice that, since $\m{dim} (\mathrm{P}) = k$, $N_\mathrm{P} = \{ 0 \}$ and $\mathcal{N}(P)$ is pointed. Furthermore,
\[
\biguplus_{N_F \in \mathcal{N}(P)} \m{int}(N_F) =  \m{C}^*,
\]
where $\m{}C^* = \{ x \in \mathbb{R}^k \colon \langle x,y \rangle \leq 0, \forall y \in \m{C} \}$ is the polar of $\m{C}$ in the $\mathcal{V}$ representation (\ref{eq:Vrepre}) of $\m{P}$  and $\biguplus$ denotes disjoint union. In particular, if $\m{C} = \{ 0 \}$,  i.e. if $\m{P}$ is a full-dimensional polytope, the cones in $\mathcal{N}(\m{P})$ partition $\mathbb{R}^k$:
\begin{equation}\label{eq:Rk}
\biguplus_{N_F \in \mathcal{N}(P)} \m{int}(N_F) = \mathbb{R}^k.
\end{equation}
We mention that, more generally, if assumption (A2) is not in force, then $N_{\m{P}}$ is a linear subspace of $\mathbb{R}^k$ of codimension $k - \m{dim}(\m{P})$.
%\begin{lemma}%\label{lem:nf}
%For any face $F$ of $\m{P}$, $a^F \in \m{relint}(N_F)$ if and only if any of the following occurs:
%\begin{enumerate}
%\item $\langle a^F, x - x_0\rangle = 0$ for all $x, x_0 \in F$ and $\langle a^F, x - x_0\rangle \neq 0$ otherwise; 
%\item $\langle a^F, x - x_0\rangle < 0$, for all $x, x_0$ such that $x_0 \in F$ and  $x \not \in F$ and $\langle a^F, x - x_0\rangle \geq 0$ otherwise;
%\item $\langle a^F, x - x_0\rangle > 0$, for all $x, x_0$ such that $x \in F$ and  $x_0 \not \in F$ and $\langle a^F, x - x_0\rangle \leq 0$ otherwise.
%\end{enumerate}
%\end{lemma}
%For any face $F$, let $x_0 \in F$. Then $\langle a^F, x - x_0\rangle = 0$ for all $x \in F$ if and only if $a^F \in \m{relint}(N_F)$.   $\langle a^F, x - x_0\rangle > 0$ for all $x \in F$ and all $x_0 \not \in F$ if and only if $a^F  \in \m{relint}( N_F)$. $\langle a^F, x - x_0\rangle < 0$ for all $x_0 \in F$ and all $x \not \in F$ if and only if $a^F  \in \m{relint}( N_F)$

Let $\m{lin}(N_F)$ denote the subspace generated by $N_F$, which is  the linear subspace spanned by the vectors $ \left( a_1, \ldots, a_{m_F} \right)$.
The following lemma shows that, for every face $F$ of the convex support, the parameter space of the extended family $\mathcal{E}_F$ can be fully described using $\m{lin}(N_F)$.

\begin{lemma}\label{lem:nonid}
The family $\mathcal{E}_F$ is non-identifiable and $\Theta_F$ is the quotient space of $\Theta$ modulo $\m{lin}(N_F)$. Furthermore, for any $\zeta \in \m{lin}(N_F)$,
\begin{equation}\label{eq:madai}
\m{rank}\left( I_F(\theta + \zeta) \right) = \m{dim}(F),
\end{equation}
where $I(\cdot)$ and $I_F(\cdot)$ denote the Fisher information matrices for $\mathcal{E}_{\m{P}}$ and $\mathcal{E}_{F}$, respectively.
\end{lemma}

%Direction of recession of sup-log likelihood.

The previous result characterizes $\Theta_F$ as the set of equivalence classes of points in $\Theta$, where $\theta_1$ and $\theta_2$ are in the same class if and only if $\theta_1 - \theta_2  \in \m{lin}(N_F)$, and the class containing $\theta \in \Theta$ is the  set
\begin{equation}\label{eq:thetaF}
\theta_F \equiv \{ \theta + \zeta \in \Theta, \zeta \in \m{lin}(N_F) \},
\end{equation}
which we call \textit{the congruence class of $\theta$ modulo $\m{lin}(N_F)$.}
Notice that if $\Theta = \mathbb{R}^k$, then $\Theta_F$ is comprised of  affine subspaces of dimension $\m{dim}(N_F) = k  -\m{dim}(F)$ parallel to $\m{lin}(N_F)$, each identifying a single distribution. In particular, when $F = \m{P}$, $\m{lin}(N_F) = \{ 0 \}$, so that $\theta_F$ is an atomic set and we recover the original, fully identifiable family $\mathcal{E}_{\m{P}}$.

\subsection{Main result}\label{sec:mainresult}

We will utilize the normal fan $\mathcal{N}(\m{P})$ to chracterize the following convergence statements:
\begin{equation}\label{eq:thm:main}
 p_{\theta_n} \rightarrow  p^F_{\theta_F}, \ a.e. \ \nu, \quad \mbox{and} \quad   \mu_n \rightarrow \mu^F \in \m{relint}(F), %\| \mu^F\|_2 < \infty,
\end{equation}
where $\mu_n = \mathbb{E}_{\theta_n}[X]$.
We take note that, because of the one-to-one correspondence between natural and mean value parameters for the families comprising  $\mathcal{E}$, the two statements imply each other.  Equation (\ref{eq:thm:main}) is of relevance as it explicitly provides various representation of the extended family $\mathcal{E}$ as the closure of the original family $\mathcal{E}_{\m{P}}$ in both natural and mean value parameterization and also in terms of almost sure limits of the densities in $\mathcal{E}_{\m{P}}$.

As a preliminary observation, we point out that (\ref{eq:thm:main}) holds true only if the parameters $\theta_n$ have diverging norms, so that $p^F_{\theta_F}$ cannot belong to $\mathcal{E}_{\m{P}}$. Formally,
\begin{lemma}\label{lem:a2}
If (\ref{eq:thm:main})  is verified, then $\| \theta_n \|_2 \rightarrow \infty$.
\end{lemma}
%We will consider sequences of natural parameters of a certain form. %More general statements are possible, see remakrs below, at the cost of additional technicalities and burdensome notation.

%describes the directions of increase of the likelihood function under natural parametrization in terms of the components of the normal fan $\mathcal{N}(\m{P})$. More specifically, we will derive some characterization of the fo %Notice, furthermore that $\mu^F \in \m{ri}(F)$ if and only if the reference distribution is supported on $F$ only.

In our main result, we establish establish sufficient conditions under which (\ref{eq:thm:main}) holds or fails, based on the cones in the normal fan of $\m{P}$.

\begin{theorem}\label{thm:main}
Consider the settings describe above and assumptions (A1)-(A4). Let $\{ \theta_n \} \subset \Theta$ be a sequence of natural parameters  satisfying $\theta_n = \eta + \rho_n d_n$, where $\{ \rho_n \}$ is a sequence of non-negative scalars tending to infinity, $\eta \in \theta^F \cap \Theta$ and $\{ d_n \}$ is a sequence of unit vectors.   
\begin{enumerate}
\item   If $\{ d_n \} \subset R$, with $R$ a  compact subset of $\m{ri}(N_F)$, then Equation (\ref{eq:thm:main}) holds
\item Conversely, if  $\{ d_n \} \subset R$, with $R$ a  compact subset of $N_F^c$, then (\ref{eq:thm:main}) fails. 
\item If $\{ d_n \} \subset R$, with $R$ a  compact subset $\left( \mathcal{N}(\m{P}) \right)^c$, then 
\begin{equation}\label{eq:mu}
\| \mu_n \|_2 \rightarrow \infty,
\end{equation}
which, in particular, implies that (\ref{eq:thm:main}) is not verified.
\end{enumerate}
 \end{theorem}
\noindent {\bf Remark} 
\begin{enumerate}
\item The assumption $\| d_n \|_2 = 1$ for all $n$ is imposed for mathematical   convenience  and does not entail any loss in generality.
\item The Theorem shows that (\ref{eq:thm:main}) will hold or fail \textit{uniformly} over  compact subsets of $\m{ri}(N_F)$, for all faces $F$ of $\m{P}$.

%over certain classes of relatively compact subsets of $\mathbb{R}^k$. %are true \textit{uniformly} over all sequences $\{ d_n \}$ of unit vectors in $\mathbb{R}^k$ such that,  for all $n$ large enough, $d_n$ is bounded away from the relative boundary of $N_F$. In particular, $\{ d_n \}$ is not required to have a limit.
\end{enumerate}

Below, we will concern ourselves with sequences $\{ \theta_n \}$ of natural parameters of a certain simplified form, as described in below.

\begin{definition}\label{def:def1}
A sequence of natural parameters $\{ \theta_n \} \subset \Theta$ is a $(\theta,d,\{\rho_n\})$-sequence if
\[
\theta_n = \theta + \rho_n d,
\]
where $\theta \in  \Theta$, $d \in \mathbb{R}^k$ and $\{ \rho_n \}$ is a sequence of non-negative numbers tending to infinity. 
\end{definition}

The restriction to $(\theta,d,\{\rho_n\})$-sequences is a strong enough condition to yield a full characterization of (\ref{eq:thm:main}),  as described in the next corollary,  and yet sufficiently mild to unveil some of the  fundamental features of the extended family $\mathcal{E}$. Furthermore, it will allow us to recast some of our results in the language of convexity theory and gain some insights on the computational aspects of calculating the extended MLE. 

%Our main result derives a full characterization of (\ref{eq:thm:main}) in terms of $(\theta,d,\{\rho_n\})$-sequences.

\begin{corollary}\label{cor:main}
Let $\{ \theta_n \}$ be a $(\theta,d,\{\rho_n\})$-sequence. %and set $\mu_n = \mathbb{E}_{\theta_n}[X]$. 
\begin{enumerate}
\item The convergence statements in (\ref{eq:thm:main}) hold if and only if $d \in \m{ri}(N_F)$.
\item If $d \not \in \mathcal{N}(\m{P})$, then (\ref{eq:mu}) is verified. 
\end{enumerate}
\end{corollary}

In essence, Corollary \ref{cor:main} characterize the extended family $\mathcal{E}$ as the compactification of the original family $\mathcal{E}_{\m{P}}$ under both natural and mean value parametrization. For the natural parametrization,  each density in $\mathcal{E}_F$ is obtained as the point-wise limit of sequences of densities parametrized by sequences of points in $\Theta$ along any direction in $\m{ri}(N_F)$ with norms diverging  to infinity. In contrast, the corresponding sequence of mean value parameters converges gracefully to the corresponding point of finite norm on the boundary of $\m{P}$. This is a striking difference between natural and mean value parametrization, which is entirely captured by the norma fan of $\m{P}$. 
See  Figures \ref{fig:9.par} and \ref{fig:9.mean} below and related discussion for more details in the context of ERG models. See also the short movies available \url{http://www.stat.cmu.edu/~arinaldo/ERG/} for a direct graphical illustration of these claims.

In the remaining of this Section, we will explore some of the consequences of Theorem \ref{thm:main} and, in particular, of Corollary \ref{cor:main}, with the goal of illustrating some of the key properties of the extended family $\mathcal{E}$.

We begin by observing that, as shown in Equation (\ref{eq:remark}) in the proof of Theorem \ref{thm:main}, if $d \not \in N_F$, the sequence of distributions parametrized by the points $\theta_n = \theta + \rho_n d$ corresponds to  distributions in the original family $\mathcal{E}_{\m{P}}$ whose mean value parameters $\mu_n \ \in \mathrm{relint}(\m{P})$ are such that $\| \mu_n \|_2 \rightarrow \infty$, with $\mu_n$ bounded away from $\m{rb}(\m{P})$. It is clear that this can occur only  if $\m{C} \neq \{ 0 \}$, i.e. if the convex support is unbounded. In fact, when $\m{P}$ is a polytope, we have $\mathcal{N}(\m{P}) = \mathbb{R}^k$ (see Equation \ref{eq:Rk}), so that Corollary \ref{cor:main} further yields that each density in the family $\mathcal{E}_F$ can be obtained as $\lim_n p_{\theta_n}$, where $\{ \theta_n \}$ is any $(\theta, \{ \rho_n \}, d )$-sequence with $\theta \in \Theta$ and $d \in \m{ri}(N_F)$. Formally,
 
\begin{corollary}\label{cor:pol}
If $\m{P}$ is a polytope, then, for any $d \in \mathbb{R}^k$, any $(\theta, \{ \rho_n \}, d)$-sequence $\{ \theta_n \}$ and any face $F$,
\[
p_{\theta_n} \rightarrow p^F_{\theta_F}, \ a.e. \ \nu, \quad \mbox{and} \quad   \mu_n \rightarrow \mu^F \in \m{relint}(F), %\| \mu^F\|_2 < \infty,
\]
if and only if  $d \in \m{ri}(N_F)$.
\end{corollary} 
In fact, our analysis of exponential random graph models of Section \ref{sec:entropy} is almost entirely an illustration of the previous Corollary.

Another  implication of Corollary \ref{cor:main} is that, when the MLE does not exist, the directions of increase of the likelihood function for a given observed sample $x \in \m{relint}(F)$ are precisely the points in the associated normal cone $N_F$.  %Using convexity theory, the result can be paraphrased as follows: 
Formally, let $X = x$ be the observed sufficient statistics and let $\ell_x: \Theta \mapsto \mathbb{R}$ be the log-likelihood function, given by $\ell_x(\theta) =  \log p_{\theta}(x)$. Then, $-\ell_x$ is a proper convex function, strictly convex  if and only if $x \in \m{ri}(\m{P})$. This follows from Lemma \ref{lem:nonid} and the well-known convexity properties of the cumulant generating function $\psi$ \citep[see, e.g.,][Theorem 1.13]{BRW:86}. Then, following \citet[][Chapter 8]{ROCKA:70}, $d \in \mathbb{R}^k$ is a direction of recession for $-\ell_x$ if
\begin{equation}\label{eq:dir}
\m{lim \; inf}_{\rho \rightarrow \infty} \ell_x(\theta + \rho d) < \infty,
\end{equation}
for one, and thus for all, $\theta \in \m{dom}(\ell_x) = \Theta$. The set of all directions of recession of $-\ell_x$ is called the recession cone of $\ell_x$. It is clear that convex functions admitting directions of recession might not achieve their  infimum at any point in their effective domain. On the account of the next result, the recession cone of $-\ell_x$ is a cone of the normal fan of $\m{P}$, almost everywhere $\nu$. 
%The recession cone of $\ell_x$ indicates the directions along which the log-likelihood function increases. %\ale{finish} approaches its supremum, without never achieving it. %Those are precisely the directions yielding the extended MLE.

\begin{corollary}\label{cor:dr}
For any observable sufficient statistics $X = x \in \m{P}$, the polyhedral cone $N_F$ is the recession cone of the negative log-likelihood function $-\ell_x$, where $F$ is the unique, possibly improper, face of $\m{P}$ such that $x \in \m{relint}(F)$.
\end{corollary}
In particular, when $x \in \m{relint}(\m{P})$, i.e. when the MLE exists, the corresponding recession cone is just the point $\{ 0 \}$ (since $\m{dim}(\m{P}) = k$), so that the negative log-likelihood function does not have any direction of recession and, therefore, its supremum is achieved at one parameter point $\widehat{\theta} \in \mathbb{R}^k$ with finite norm, namely the MLE. On the other hand, when the MLE is nonexistent, the likelihood function increases for any sequence of natural parameters   with norm diverging to infinity along any direction $d \in N_F$, where $N_F$ is the normal cone of the face of $\m{P}$ containing the observed sufficient statistics in its relative interior. %In contrast the  estimate of teh mean value parameter corresponding sequence of mean values converges gracefully to the appropriate point on the boundary of $\m{P}$. 

%\ale{HERE: During the preparation of the paper, we learned of similar results by \cite{GEYER:08}, which rely on a different approach closer to convexity theory.}

%Then, combining Corollary \ref{cor:dr} with Lemma \ref{lem:nonid}, we can conclude that, when $x \in \m{relint}(F)$, the extended MLE will be an affine subspace of dimension  $\m{dim}(N_F)$, as shown next.

We note that Corollary \ref{cor:dr} could be stated in a more general form. Indeed, for any $\xi \in \m{P}$, letting $\ell_\xi \colon \Theta \mapsto \mathbb{R}$ be given by
\[
\ell_\xi(\theta) = \langle \theta , \xi \rangle - \psi(\theta),
\]
it can be verified that the proof of Corollary \ref{cor:dr} still holds with $\ell_x$ replaced by $\ell_\xi$. Though theoretically relevant, this fact has little practical value.

During the preparation of the paper, we learned of similar results in \cite{GEYER:08}, which are based on the characterization of the convex support in term of the tangent cones and normal cones. While his analysis applies to more general classes of exponential families, our results are more refined, as we take full advantage of the polyhedral assumption and establish a more direct connections between the extended families and the cones in the normal fan.

By combining the results derived so far, we next show that, when $x \in \m{relint}(F)$, the extended MLE will be the affine subspace of dimension  $\m{dim}(N_F)$ given by $\widehat{\theta}_F$, where $\mathbb{E}_{\widehat{\theta}_F} = x$. Though not entirely a new result \citep[see][Chapter 6]{BRW:86},  our proof and the characterization of $\widehat{\theta}_F$ in terms of $N_F$ is novel.

%In particular, the extended MLE will be found traveling along any such directions. %form a cone of dimension $\m{dim}(N_F)$ (or $\m{F}$?).

\begin{corollary}\label{cor:uno}
Let $x \in \m{relint}(F)$ and $\widehat{\theta}_F$ be the congruence class of $\theta$ modulo $\m{lin}(N_F)$ such that $\mathbb{E}_{\widehat{\theta}_F}[X] = x$. Then,
\[
\sup_{\theta \in \Theta} p_{\theta}(x) = p^F_{\widehat{\theta}_F}(x).
\]
\end{corollary}
%\ale{NOOO} We remark that, by Lemma \ref{lem:nonid}, $\ell_x(\theta)$ is constant for all $\theta \in \widehat{\theta}_F$. Using convexity theory again  \citep[see][Chapter 8]{ROCKA:70}, $\m{lin}(N_F)$ is the constancy space of $\ell_x$, so that,  as $\m{P}$ is full-dimensional, the MLE exists if and only if the constancy space is the trivial subspace $\{ 0 \}$. 
%This last corollary implies that the maximum likelihood estimate of $\theta$ when $x$ is the observed sufficient statistic is a linear subspace of dimension. 
%Part 2. of Theorem \ref{thm:main} further shows that nonexistence of the MLE of the natural parameter implies that the likelihood function increases along sequences of parameters with norm diverging to infinity. In contrast, 

For completeness, we conclude this section by linking our discussion with alternative characterizations of the closure of the family $\mathcal{E}_{\m{P}}$ existing in the literature, which could be easily obtained using Theorem \ref{thm:main} \citep[see, in particular,][]{CS:01,CS:03,CS:05,CS:08}.
\begin{corollary}\label{cor:tre}
For any $(\theta, \{ \rho_n \}, d)$-sequence $\{ \theta_n \}$ with $d \in \m{ri}(N_F)$,%, and any $\zeta \in \m{lin}(N_F)$,
\begin{enumerate}
\item[i)] $P_{\theta_n} \stackrel{\m{TV}}{\rightarrow} P^F_{\theta_F}$, where $ \stackrel{\m{TV}}{\rightarrow}$ denotes convergence in total variation;
\item[ii)] $\lim_n K(P^F_{\theta_F},P_{\theta_n}) = 0$, where $K(P,Q)$ is the Kullback-Lieber divergence of $P$ from $Q$; 
\item[iii)] $P_{\theta_n} \Rightarrow P^F_{\theta_F}$, where the $\Rightarrow$ denotes convergence in distribution.
%\item[iv)] $\Lambda_{\min}(I(\theta_n)) \rightarrow \Lambda_{\min}(I_F(\theta+\zeta)) = 0$, where $I(\cdot)$ and $I_F(\cdot)$ denote the Fisher information matrices for $\mathcal{E}_{\m{P}}$ and $\mathcal{E}_{F}$, respectively, and $\Lambda_{\min}(A)$ is the smallest eigenvalue of $A$. 
\end{enumerate}
\end{corollary}

\subsection{Computational considerations}\label{sec:comput}

Based on our findings, we can make a few observations regarding the computational difficulties of finding the extended MLE, some of which are exemplified in the next result.

\begin{corollary}\label{cor:due}
Let $\{ \theta_n \}$ be a $(\theta, \{ \rho_n\}, d)$-sequence, with $d \in \m{ri}(N_F)$. Then, for every $\zeta \in \m{lin}(N_F)$, 
\begin{equation}\label{eq:se}
I(\theta_n) \rightarrow I_F(\theta + \zeta),
%\Lambda_{\min}(I(\theta_n)) \rightarrow \Lambda_{\min}(I_F(\theta + \zeta)) = 0,
\end{equation}
where convergence is pointwise. %, and $\Lambda_{\min}(A)$ is the smallest eigenvalue of $A$. 
\end{corollary}
From the corollary and  equation (\ref{eq:madai}), we can infer that, when the MLE does not exist, maximizing the log-likelihood function using the Newton Rapson method, as well as virtually any other fastest ascent methods, may fail due to numerical instabilities. In fact, the Newton Rapson algorithm proceeds by finding a sequence $\{ \theta_n \}$ of natural parameters along which $\ell_x$ increases most rapidly. At each step of the procedure, the next point in the sequence is determined by the direction of fastest ascent of $\ell_{x}$, given by the inverse of the Hessian, e.g. by the inverse of $I(\theta_n)$.  However, for all $n$ large enough, these matrices will be  badly conditioned, since, at the optimum, the Fisher information matrix is not invertible (see equation \ref{eq:madai}). In addition, especially  when the observed statistics $x$ belong to the relative interior of a face of small dimension, these singularities can be dramatic, not to mention the fact that, unless $x$ lies on a the relative interior of a facet, there is an infinite number of directions along which the likelihood function increases.
It is apparent that these problems are even more accentuated in high-dimensional settings or whenever the data are sparse. 
From the statistical standpoint, equations (\ref{eq:se}) and  (\ref{eq:madai}) further imply that, when the MLE does not exist, the standard error may be quite large (in the limit, infinite), and that the number of degrees of freedom should be adjusted to reflect the non-estimability of some parameters. As a result, any hypothesis testing or model selection procedure that rely solely on these estimates should be regarded, at the very least, unreliable.
Based  on these considerations, it is clear that not only is the task of computing the extended MLE particularly daunting, but  the statistical interpretation of these quantities is also rather delicate.

 %an observed point on $\m{relint}(F)$ form a cone of dimension $\m{dim}(N_F)$ (or $\m{F}$?).

We refer the reader to \cite{GEYER:08} and \cite{ALE:06b} for different algorithmic approaches on computing the extended MLE for certain types of exponential models with polyhedral support for which a $\mathcal{V}$ representation of $\m{P}$ of the form (\ref{eq:P})  or (\ref{eq:Vrepre}) is either available or easily computable. We remark that, in order to determine the extended MLE, it is necessary not only to have an explicit representation of $\m{P}$  but, in addition,  to be able to have in closed form the log-partition functions $\psi^F$, for each face $F$. A class of models for which both conditions are satisfied is the class of the log-linear models. If this information is not available, one may resort to MCMC techniques for computing the MLE or a pseudo-MLE, as for the class of models to be described in the next section. See \cite{GT:92}, and \cite{H:03}, \cite{SNIJ:02}, \cite{WR:04}, \citet{STATNET}, \cite{HGH:08} and references therein.

As a final comment,  we point out that, while computing the extended MLE is very often a hard problem, deciding whether the MLE exists is typically more feasible, and can be accomplished using linear programming, provided an explicit representation, namely a $\mathcal{H}$ or a $\mathcal{V}$ representation, of $\m{P}$ is available. See Appendix C for details and also \cite{ERIKSSON:06} for an application to hierarchical log-linear models.

\section{Application to Exponential Random Graph Models}\label{sec:ERG}

We now apply some of the results from  the previous section to the class of exponential random graph models. The motivation for our choice is the attempt to explain certain features of ERG models that have been observed empirically and have been  collectively labeled as degeneracy \citep[see, e.g.,][]{H:03}. Our point of view is simply that there is nothing degenerate or unusual about these models, whose behavior can in fact be explained in a direct way using the properties of exponential families with  polyhedral support as described in Section \ref{sec:mainresult}.

Our arguments rely on a thorough analysis of one ERG model, described below in Section \ref{sec:one}, and on graphical renderings of Corollary \ref{cor:pol}, which we find particularly effective and elucidative of our results. We looked at a variety of other ERG models on 7,8 and 9 nodes, using different choices of the network statistics described below, and arrived to the same kind of conclusions we are about to present. 

Finally, we would like to emphasize that, as the log-partition function is not available in closed form, an exact analysis of ERG models on larger graphs is almost impossible. This is due to the need to enumerate all possible graph with a given number of nodes in order to evaluate that function, a task whose computational computationally becomes prohibitive very rapidly as the number nodes grow;  see  Equation (\ref{eq:ngraphs}) below and Table \ref{tab:tab1}.

\subsection{Introduction to ERG models}\label{sec:ERGintro}

There is an extensive literature of ERG models and their use in social network analysis. A partial but representative list of references  is: \cite{HL:81}, \cite{FS:86}, \cite{WP:96}, \cite{WR:04}, \cite{RA:07,R:07} and references therein. Below we briefly describe the settings for ERG models, in order to make explicit the connections with the material in the previous sections.

Consider the set $\mathcal{G}_g$ of all possible simple, i.e. unweighted, undirected and without loops, graphs  on $g$ nodes. Every such graph $x$ can be described by a 0-1 symmetric $g \times g$ adjacency matrix, whose $(i,j)$-th entry is $1$ if there exists an edge between the nodes $i$ and $j$ and $0$ otherwise. Thus,  $x$ can be represented as a ${g \choose 2}$-dimensional 0-1 vector.
The cardinality of $\mathcal{G}_g$ grows super-exponentially in the number of nodes $n$, namely
\begin{equation}\label{eq:ngraphs}
|\mathcal{G}_g| = 2^{{g \choose 2}},
\end{equation}
so that network modeling entails constructing probability distributions over very large discrete spaces (see Table \ref{tab:tab1}).

\begin{table}
\begin{center}
\begin{tabular}{|c|c|r|}\hline
Number of nodes: $g$ & Number of edges: ${g \choose 2}$ & Number of graphs: $|\mathcal{G}_g|$\\
\hline
7 & 21 & $2,097,152$ \\
8 & 28 & $268,435,456$ \\
9 & 36 & $ 68,719,476,736$ \\
10 & 45 & $35,184,372,088,832$\\\hline
\end{tabular}
\end{center}
\caption{Some information about the complexity of some ERG models on small graphs.}
\label{tab:tab1}
\end{table}

Let $T \colon \mathcal{G}_g \mapsto \mathbb{R}^k$ be a vector valued function of \textit{network statistics} quantifying the key features of interest of a given observed graph. 
In this article we are mostly concerned with ERG models arising from network statistics that capture rather general and aggregate features of the network. Typical examples of such statistics are  \citep[see, e.g.,][for more details]{G:07}:	
\begin{enumerate}
\item the number of edges: $ \sum_{i < j} x_{ij}$
\item the number of triangles: $\sum_{i < j < h}x_{ij}x_{jh}x_{ih}$
\item the $k$-degree statistic: $D_k(x) = \sum_{i=1}^g 1\{ d_i = k\}$, where $d_i = \sum_j x_{ij}$ is the degree of the $i$-th node and $0 \leq k \leq n-1$;
\item the number of $k$-stars: $\sum_{i=k}^{g-1} {i \choose k} D_i(x) $, $2 \leq k \leq n-1$, i.e. the number of distinct edges that are incident to the same node, where $D_i(x)$ is the $i$-the degree statistic given above;%ordered $(k+1)$-typles of nodes such that the first node is connected to the remaining $k-1$ nodes;
\item the alternating $k$-star statistic
\[
\sum_{i=2}^{g-1} (-1)^{i-1}\frac{S_i(x)}{\lambda^{2-i}},
\]
where $\lambda$ is a positive parameter.
\end{enumerate}

For all modeling purposes, these network statistics are effectively regarded as sufficient statistics and, by Koopman-Pitman-Darmois theorem, the resulting exponential family of distributions provides a  convenient statistical model for $\mathcal{G}_g$.  Formally, given a set of network statistics in the form of a $k$-valued function $T(\cdot)$ on $\mathcal{G}_g$, the ERG model $\mathcal{P} \equiv \{ Q_{\theta}, \theta \in \Theta \subseteq \mathbb{R}^k\}$ is the exponential family of probability distributions  over $\mathcal{G}_g$  with natural sufficient statistics $T(x)$ and base measure $\mu$ given by the counting measure on $\mathcal{G}_g$. Thus, for $\theta \in \Theta$, the density of $Q_{\theta}$ with respect to $\mu$ is
\[
\frac{d Q_\theta }{d \mu} (x) = q_{\theta}(x) = \exp \{ \langle T(x), \theta \rangle - \psi(\theta) \} = Prob\{ X = x \}.
\]

Let $\mathcal{T} = \{ t \in \mathbb{R}^k \colon t = T(x), x \in \mathcal{G}_g \}$ be the range of $T(\cdot)$ and $\nu$ the measure on $\mathcal{T}$  induced by $\mu$, namely
\[
\nu(t) = \mu \{ x \in \mathcal{G}_g \colon T(x) = t \} = \left| \{ x \in \mathcal{G}_g \colon T(x) = t \} \right|, \quad t \in \mathcal{T}.
\]
Then, the distribution of $T(X)$ belongs to the exponential family of distributions on $\mathcal{T}$ with base measure $\nu$,  natural parameter space $\Theta$ and densities
\[
p_{\theta}(t) = \exp \{ \langle t, \theta \rangle - \psi(\theta) \}, \quad \theta \in \Theta.
\]
Furtheremore, because of the discreteness of the problem, 
\[
Prob(T(X) = t )  = \int_{\{ x \in \mathcal{G}_g \colon T(x) = t \}} q_\theta(x) d \mu(x) = \sum_{\{ x \in \mathcal{G}_g \colon T(x) = t \}} q_\theta(x) = p_{\theta}(t) \nu(t). %\exp \{ \langle t(x), \theta \rangle - \psi(\theta) \}.
\]
Provided that the network statistics are affinely independent, as it is the case for the examples given above and as it can always be assumed through reduction to minimality,  the convex support $\m{P} = \m{convhull}(\mathcal{T})$ is a $k$-dimensional ploytope. Recalling that $\nu$ is finite, it is easy to see that the assumptions (A1)-(A4) of Section \ref{sec:setting} are verified, and the theory developed above applies.%since $\mathcal{T}$ is a finite set in general position,

Despite its simplicity and interpretability, we need to emphasize that ERG modeling based on simple, low dimensional network statistics such as the ones described above can be rather coarse. In fact, those ERG models are invariant with respect to the relabeling of the nodes and even to changes in the graph topologies, depending on the network statistics themselves. As a result, they do not specify distributions over graphs per se, but rather distributions over large classes of graphs having the same network statistics.   Consequently, as we repeatedly observed in our experiments and as elucidates in the example we are about to present, it may very well be the case that many graphs having very different topologies still belong to the same class and, therefore, are considered as equivalent. While this feature may be well suited  for defining distributions over large thermodynamic ensambles in statistical physics, its use in other contexts in which the nodes are not interchangeable may be questionable. 
This is certainly not a common feature of all ERG models: for example, the $p_1$ model by \cite{HL:81} and the Markov graphs by \cite{FS:86} are based on much finer network statistics whose dimension, unlike the aggregate statistics reported above, increases with the size of the network. These more complex models represent explicitly distributions of individual networks rather than of classes on networks:  both $p_1$ and Markov graph models are log-linear models over the probability of edges \citep[see][]{FW:81}. However, they also present difficulties. In fact, not only is the MLE not likely to exist if the observed network is even moderately sparse, but the asymptotics of these models as $g$ grows remains unknown \citep[see, e.g.][for some comments on $p_1$ models]{H:81}. While the theory developed in the previous sections apply to all ERG models, our analysis below is more directly relevant to models arising from simpler network statistics quantifying macroscopic properties of the network.

\subsection{Our Running Example}\label{sec:one}

We will be using throughout the example of a ERG model on $\mathcal{G}_9$ with two-dimensional network statistic $T(x) = (T_1(x),T_2(x)) \in \mathbb{N}^2$, where $T_1(x) $ is the number of edges and $T_2(x)$ is the number of triangles.  Note that this model is not hierarchical in the sense of~\cite{bishop:fienberg:holland:1975} and \cite{lee:nelder:1996}, since we do not include the network statistic for the number of 2-stars, which lie intermediate to edges and triangles.  The lack of hierarchical model structure affects the interpretation of the exponential family parameters corresponding to $T(x)$ but turns out not to be the cause of the degeneracies we illustrate.  We have actually produced similar results for models which are fully hierarchical, but the results are easier to demonstrate in the context of this ERG model with a  two-dimensional network statistic. 

The number of distinct graphs for this  $\mathcal{G}_9$  example is $2^{36}$, while the number of two-dimensional distinct network statistics is only ${9 \choose 2} {9 \choose 3} = 444$. 
The natural parameter space is the entire $\mathbb{R}^2$. The support of the distribution of $T(X)$ is shown in Figure \ref{fig:9.supp}. The convex support for the induced family of distributions of network statistics is a polygon with 6 edges, whose boundary is depicted with the red solid line. Out of the possible $444$ points, $29$ actually lie on the boundary. The induced base measure $\nu$ for this family, i.e. the frequencies of each possible pair of network statistics, is indicated by the color shading of the circles. The maximal value of $\nu(t)$ is $1,876,664,161$, the median value is $2,741,130$, while the first and third quartiles are  $545,265$ and $79,674,084$, respectively. Figure \ref{fig:9.quant} shows a plot of the empirical quantile function for $\nu(t), t \in \mathcal{T}$, which indicates that few network configurations are much more frequent than others.

\begin{figure}[ht]
\centering
\includegraphics[width=4in]{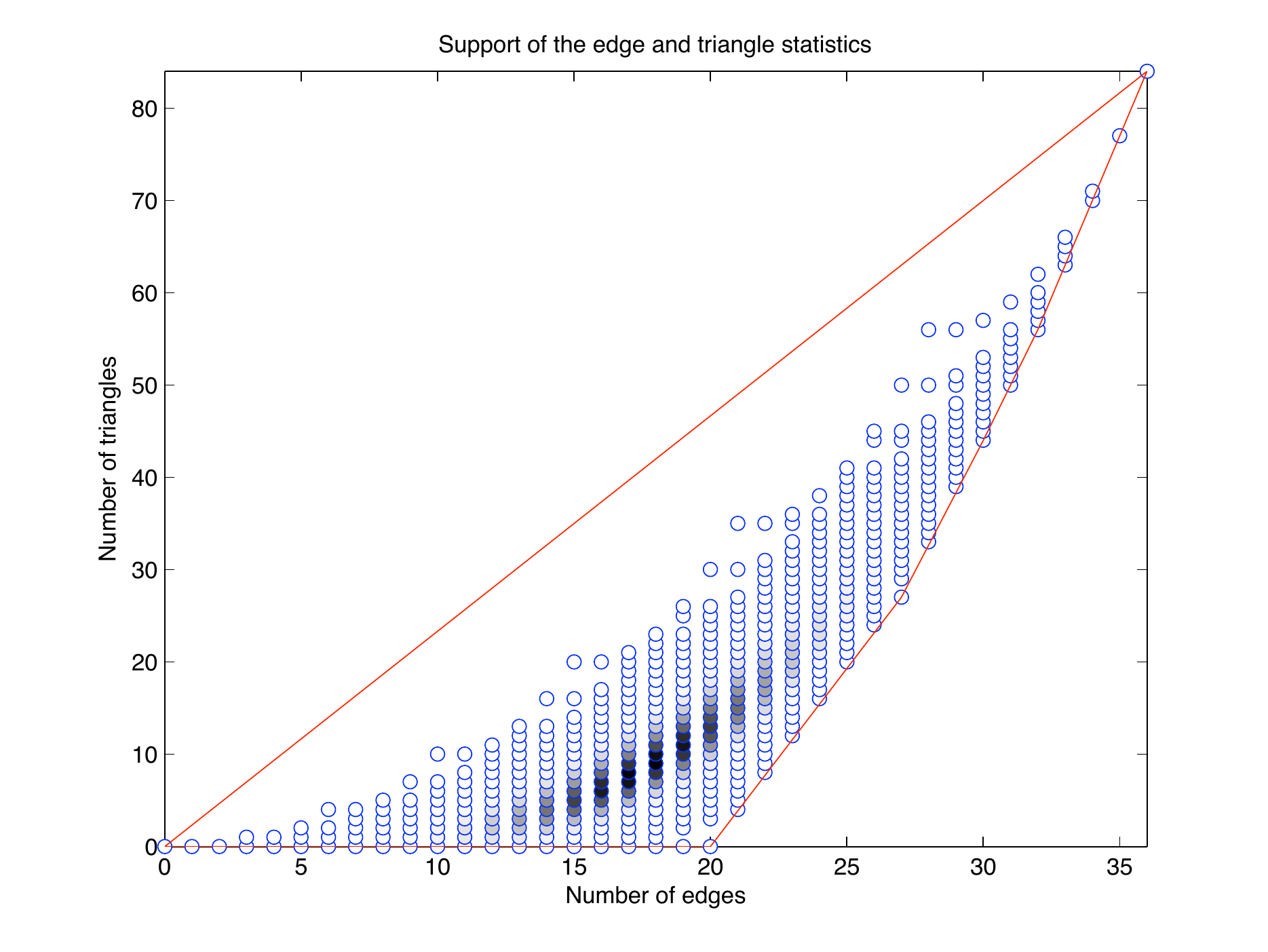}%{madai_or.eps}%{graph9_supp_sqrt.eps}%{madai_or.eps}%graph9_supp_colored_different.eps}	
\caption{Support of the distribution of the network statistics for the ERG model on  $\mathcal{G}_9$ described in  Section \ref{sec:one}. The color shading indicates the squared root of the relative frequency of each point, namely $\nu(t)$ (darker colors correspond to higher-frequency values of $t$). The solid red line is the boundary of the convex support.}
\label{fig:9.supp}
\end{figure}

\begin{figure}[ht]
\centering
\includegraphics[width=4in]{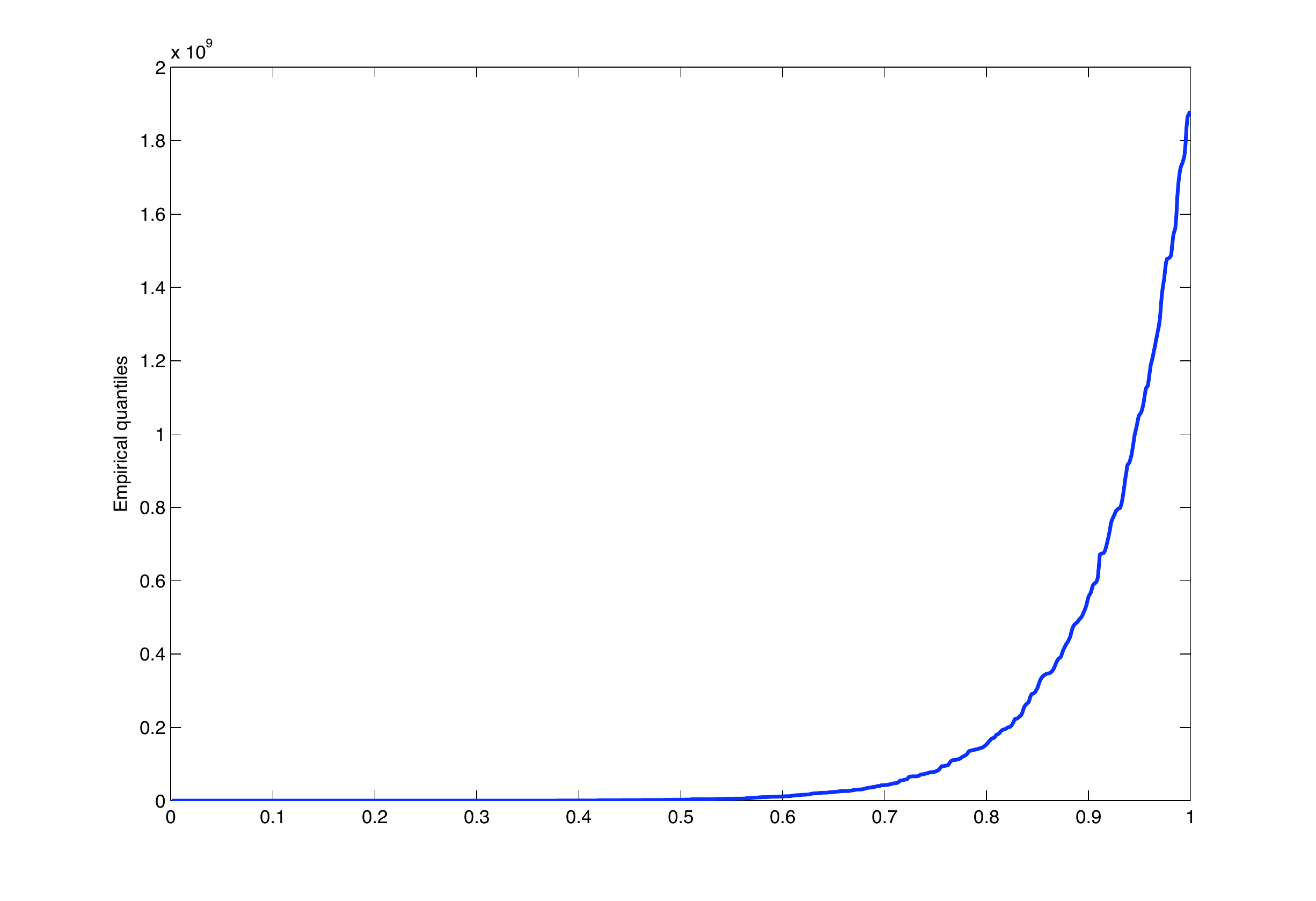}%{graph9_supp_sqrt.eps}%{madai_or.eps}%graph9_supp_colored_different.eps}	
\caption{Empirical quantiles of the values $\{ \nu(t), \colon t \in \mathcal{T}\}$ for the base measure of the family described in Section \ref{sec:one}.}%Support of the distribution of the network statistics for the ERG model on  $\mathcal{G}_9$ of Example \ref{sec:one}. The color shading indicates the squared root of the relative frequency of each point, namely $\nu(t)$ (darker colors correspond to higher-frequency values for $t$). The solid red line is the boundary of the convex support.}
\label{fig:9.quant}
\end{figure}

\subsection{Degeneracy}\label{sec:degeneracy}

%\ale{try to define degeneracy}

The notion of \textit{degeneracy} is central to ERG modeling, and has been investigated  in various forms in the more recent literature. See \cite{SNIJ:02}, \cite{R:07}, \cite{RA:07} and, in particular, \cite{H:03} and \cite{HGH:08}, just to mention a few. 
Degeneracy refers quite broadly to a variety of  features, typically undesirable and surprising, of ERG models that have been observed empirically. In the literature, degeneracy (or near degeneracy) is  used to describe any of the following, often interrelated, phenomena:
\begin{enumerate}
\item when a combination of ERG parameters $\theta$ implies that only a very small number of distinct graphs have substantial non-zero probabilities; in the most extreme cases, these configurations are the empty graph or the fully connected graph;
\item when, for a certain combination of ERG parameters $\theta$, the density function $p_{\theta}$ has multiple, clearly distinct, modes, and there are only very few network configurations that have non-zero probabilities, often radically different from each other;
\item when the MLE of $\theta$ is nonexistent or hard to obtain, or the MCMCMLE of $\theta$ fails to converge or appears to converge extremely slowly;
\item when the estimate of $\theta$ would make the observed network configuration very unlikely.
\end{enumerate}

Each of the situations just described offers  strong evidence of misspecifcation or, at the very least, of the inability of the model to describe in a realistic fashion the observed network. To our knowledge, \cite{H:03} is the only attempt to characterize degeneracy in a theoretical way, at least the kind of degeneracy yielding unstable maximum likelihood estimates, with emphasis on MCMCMLE methods. %Handcock's analysis was the motivated us to  further investigate, through the theory of exponential families, degeneracy in ERG models.

\begin{figure}[h!]
\centering
\begin{tabular}{c}
{\bf (a)}\\
 \includegraphics[width=4in]{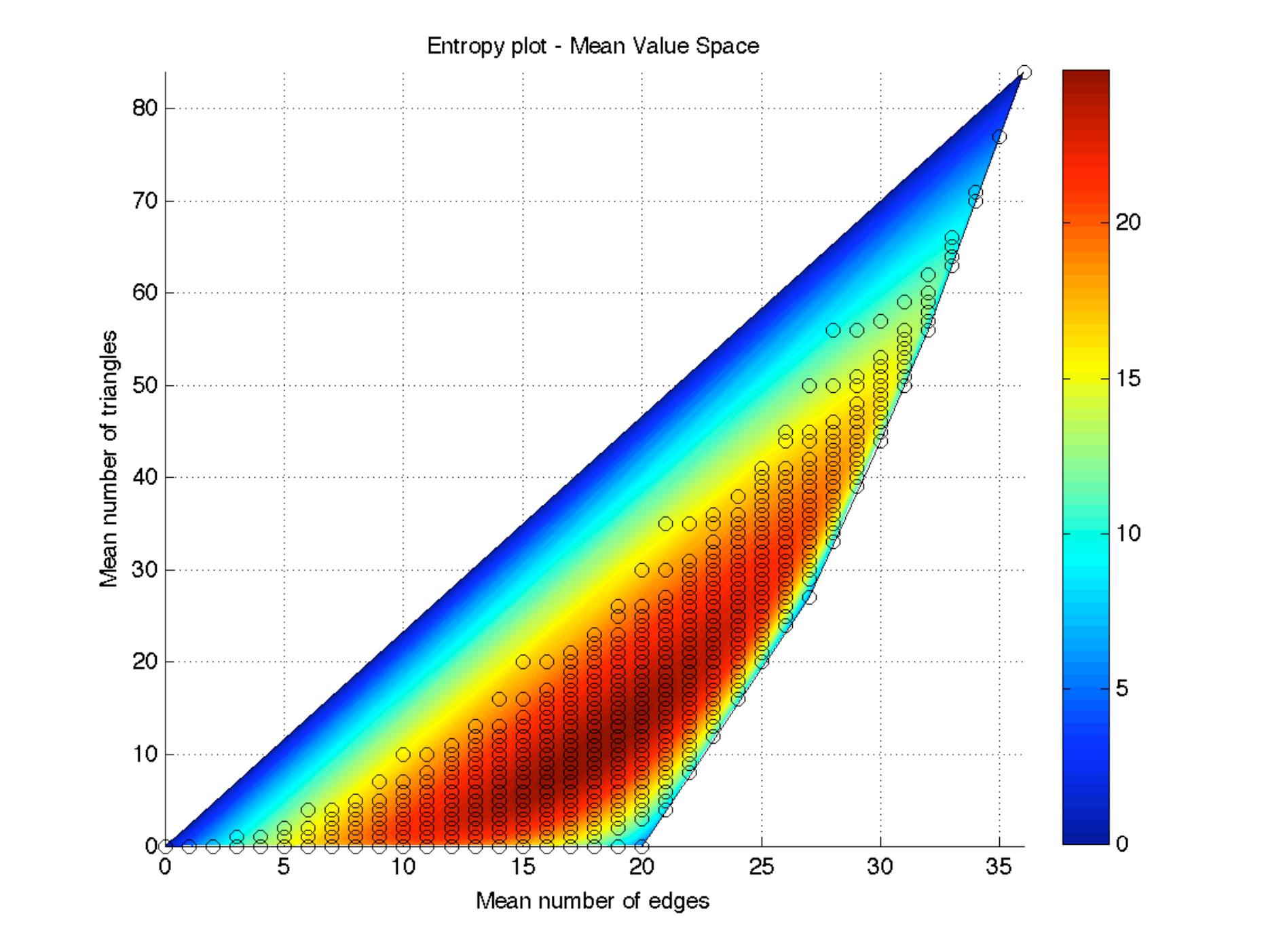}\\%{graph9_mean_par_colors.eps}\\
 {\bf (b)}\\
 \includegraphics[width=4in]{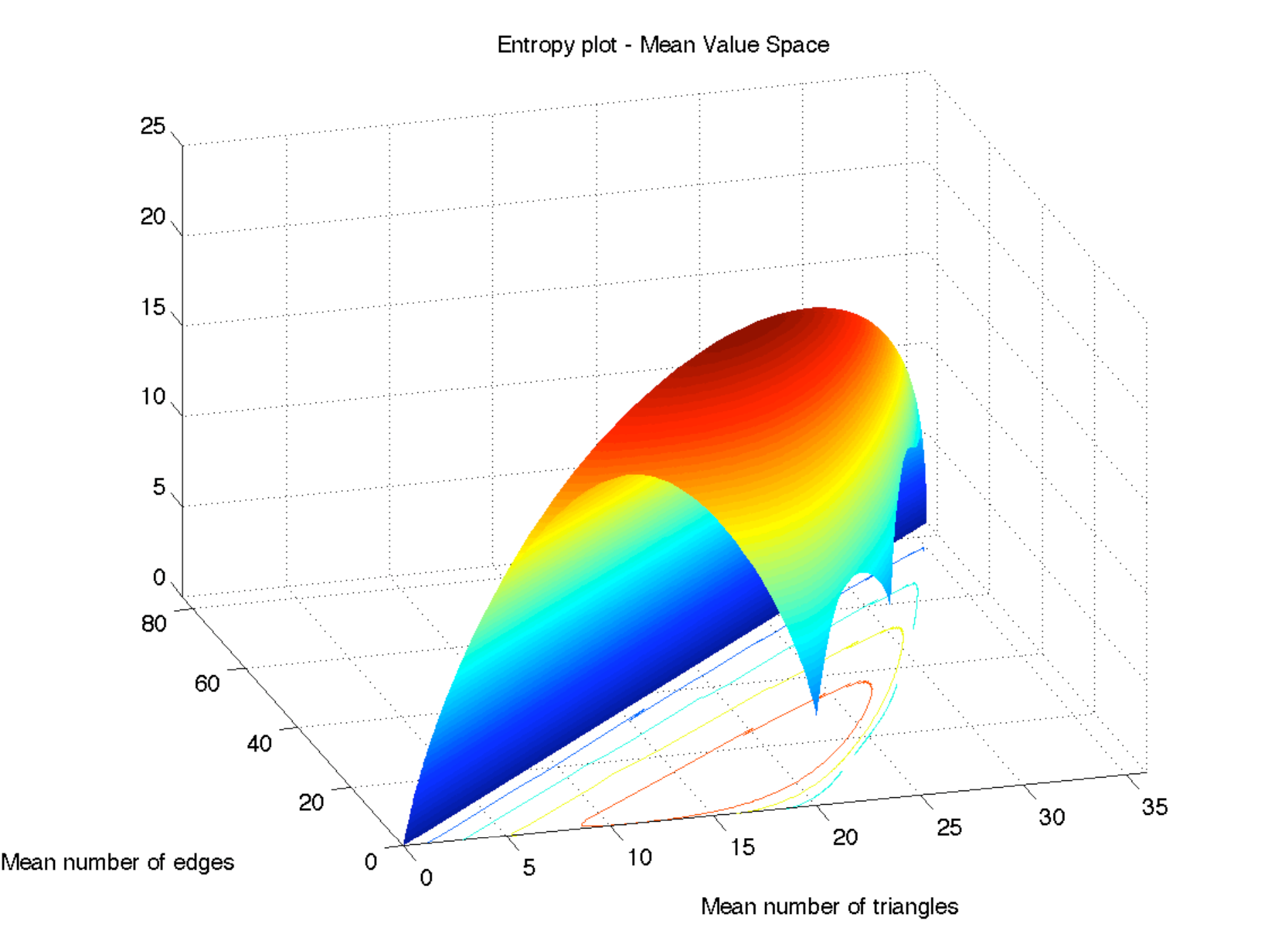}\\
\end{tabular}
\caption{Plots of the entropy function $V(\cdot)$ under mean value parametrization for the ERG model of Section \ref{sec:one}. Part {\bf a)}: 2-dimensional plot over the convex support $\m{P}$; the points correspond to the support of the family. Part {\bf b)}: surface plot.}
\label{fig:9.mean}
\end{figure}

\subsection{Degeneracy via Entropy Functions}\label{sec:entropy}
We based our analysis on a basic observation: a common feature of all the various instances of degenerate ERG models is that the corresponding distributions are highly concentrated on network configurations associated to a small number of network statistics. Therefore, in order to capture the overall degree of concentration of the family $\mathcal{P}$, we turn to Shannon's entropy function, the rationale being that degenerate models have lower entropy.
%rather than Rather than computing the probability of a a degenerate graph, we instead consider Shannon's entropy of the distributions as a function of the natural parameter and natural sufficient statistics. 

Shannon's entropy function $S \colon \Theta \rightarrow \mathbb{R}$ is defined as 
\[
S(\theta) =  - \sum_{x \in \mathcal{G}_g} q_\theta(x) \log q_{\theta}(x)  = - \sum_{t \in \mathcal{T}} p_\theta(t) \log p_{\theta}(t) \nu(t),
\]
where the second summation involves a much smaller number of terms. Notice that, for every $\theta \in \Theta$,
\[
0 \leq S(\theta) \leq {g \choose 2} \log 2,
\]
the lower and upper bounds corresponding to a degenerate distribution with point mass at one graph, and to the uniform distribution over $\mathcal{G}_g$ (which is within the family if $\nu(t)$ is constant across $\mathcal{T}$ and $\theta = 0$), respectively. Furthermore, as $\psi$ is an analytic function of $\theta$, for every $\theta \in \Theta$, $S(\theta)$ is a smooth function of $\theta$.

%Since $\lim_{x \rightarrow 0} x \log x = 0$, $\lim_n S(\theta_n)$ exists for every sequence of natural parameters $\{ \theta_n\} \subset \mathbb{R}^k$. 
Noting that $\lim_{x \rightarrow 0} x \log x = 0$ and using the fact that $S(\theta)$ is bounded, by the dominated convergence theorem Corollary  \ref{cor:main} yields that,  for every $(\theta, \{ \rho_n\}, d)$-sequence $\{\theta_n \}$ with $d \in \m{ri}(N_F)$, 
\begin{equation}\label{eq:entr}
\lim_n S(\theta_n) = S_F(\theta_F) \equiv - \int_{\mathcal{T}} p_{\theta_F}(t) \log p_{\theta_F}(t) d \nu_F(t),
\end{equation}
for every face $F$ of $\m{P}$. 

On the other hand, because of the correspondence between natural and mean value parameters, the entropy function can be equivalently represented as a function over $\m{P}$. More precisely, we define $V \colon \m{P} \mapsto \mathbb{R}$ as follows: if  $\mu \in \m{relint}(\m{P})$,
\[
V(\mu) = S(\theta),
\]
where $\mu = \nabla \psi(\theta)$, while, for $\mu_F \in \m{relint}(F)$,
\[
V_F(\mu_F) = S_F(\theta_F),
\]
where  $\mu_F = \nabla \psi^F(\theta_F)$. Thus, if $\{ \theta_n \}$ is a $(\theta, \{ \rho_n\}, d)$-sequence with $d \in \m{ri}(N_F)$ and if $\mu_n = \mathbb{E}_{\theta_n}[T(X)]$, from Equation (\ref{eq:entr}) we obtain that
\[
\lim_n V(\mu_n) = V_F(\mu_F),
\]
where $\mu_F = \lim_n \mu_n$, with $V(\mu)$  a smooth function of $\mu$. Thus, we conclude that $S(\cdot)$ and $V(\cdot)$ have homeomorphic graphs and, therefore, they convey the same information.

Below, we use both entropy functions to illustrate the theory developed in Section \ref{sec:mainresult} and to provide some characterizations of degeneracy.

\begin{figure}[h!]
\centering
\begin{tabular}{c}
%{\bf a)} &   & {\bf b)}\\
{\bf (a)}\\
 \includegraphics[width=4in]{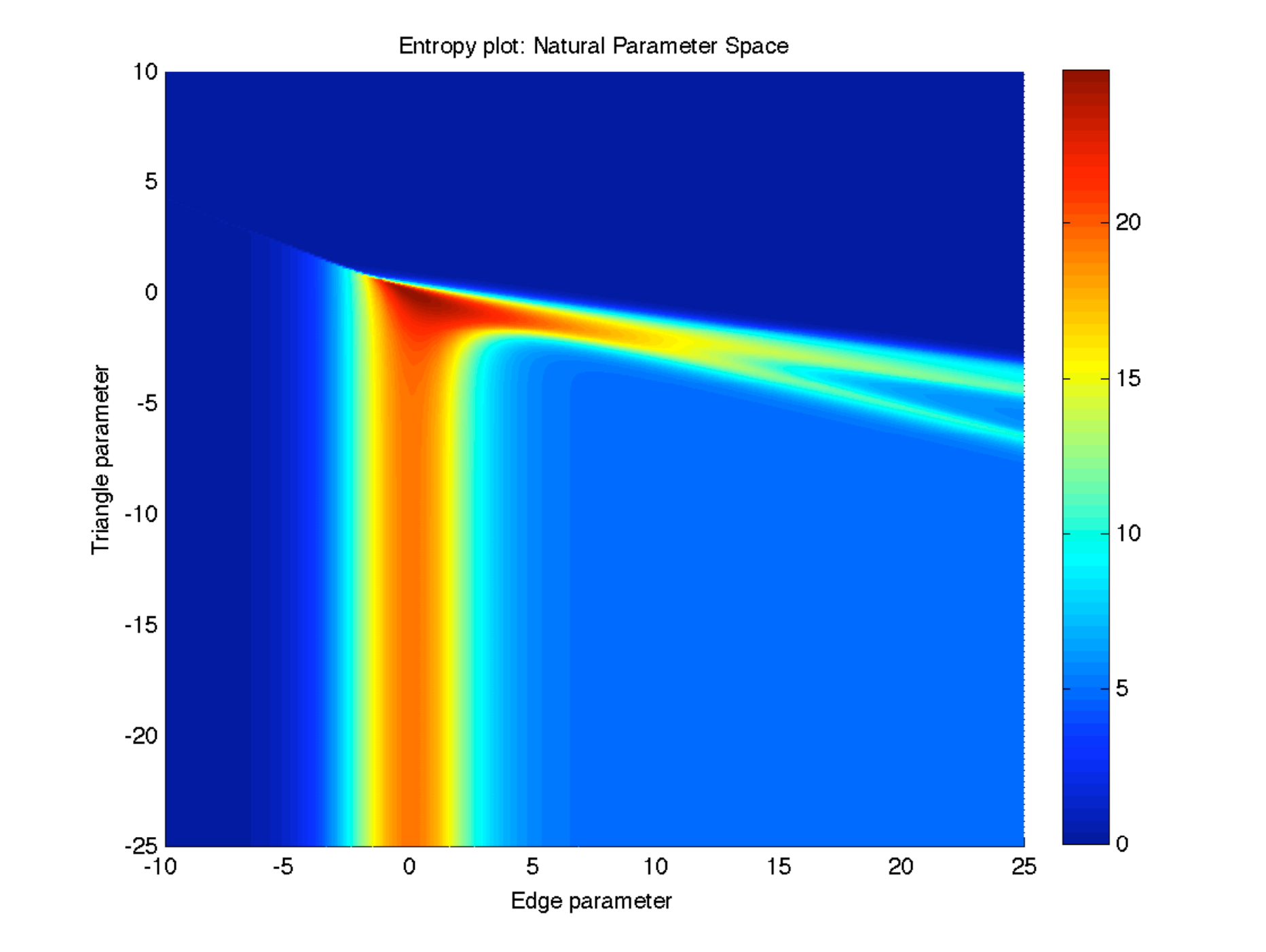}\\
{\bf (b)}\\
\includegraphics[width=4in]{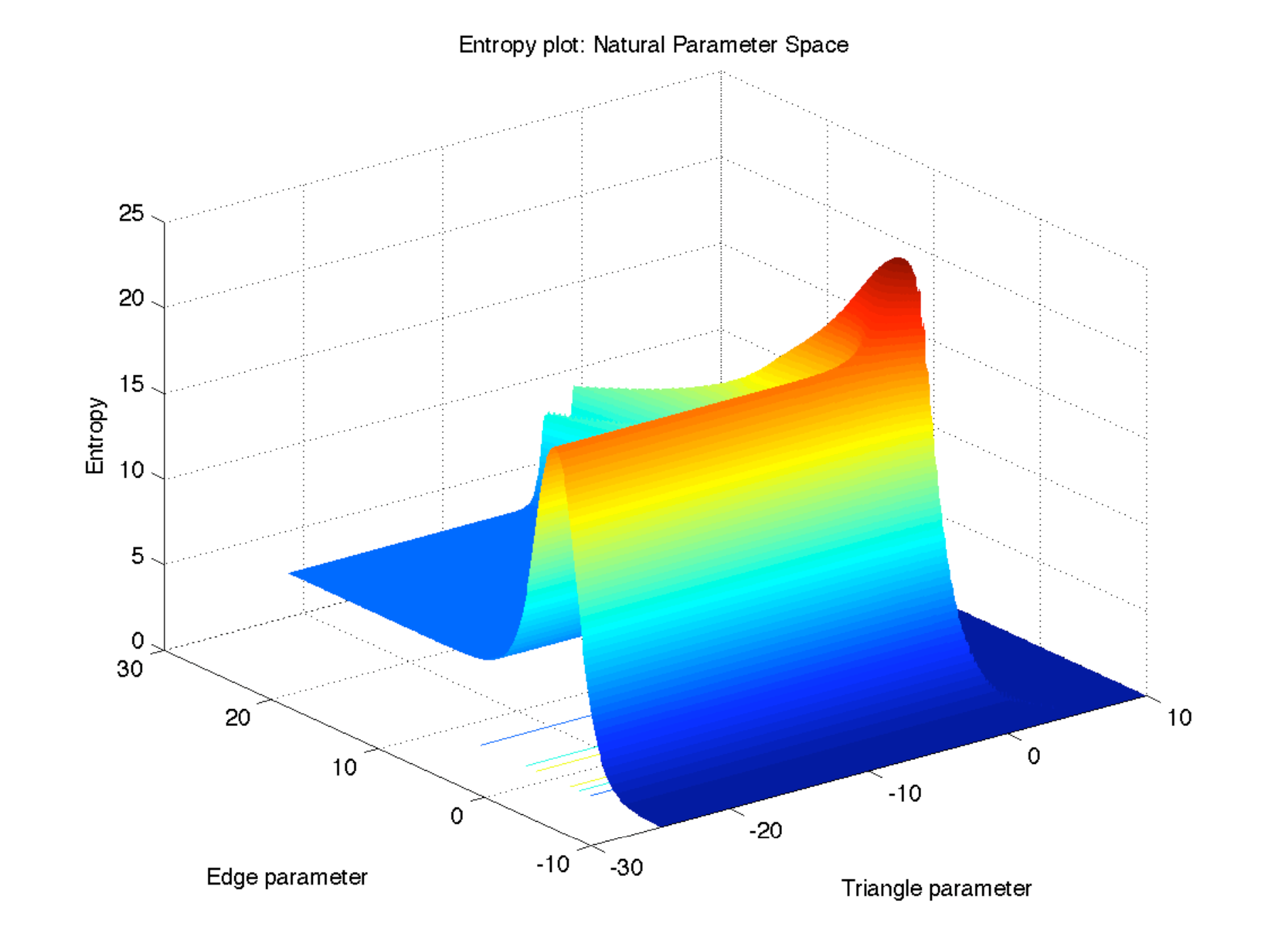} \\
\end{tabular}
\caption{Plots of the entropy function $S(\cdot)$ under natural parametrization for the ERG model of Section \ref{sec:one}. Part {\bf a)}: 2-dimensional plot over a square of the natural parameter space. Part {\bf b)}: surface plot.}
\label{fig:9.par}
\end{figure}

We start with Figures \ref{fig:9.mean} and \ref{fig:9.par}. The latter displays the entropy function $S(\theta)$ for the ERG model on $\mathcal{G}_9$ with network statistic taking values in $\mathbb{N}^2$, as described in Section \ref{sec:one}, and for values of $\theta$ in the rectangle $[10,25] \times [-25,10]$. The equivalent entropy function over the mean value space $V(\mu)$ is displayed in Figure \ref{fig:9.mean}, for the mean value parameters  $\{ \mu \colon \mu = \nabla \psi(\theta), \theta \in [10,25] \times [-25,10]\}$.
%{\bf a)} and {\bf b)}
Figures  \ref{fig:9.par} and \ref{fig:9.mean} offer two equivalent views of the exponential family $\mathcal{P}$ via the entropy functions $S(\theta)$ and $V(\mu)$. The mean value view in Figure \ref{fig:9.mean} is straightforward to interpret: the entropy function is a well behaved, strictly concave function that changes smoothly as the mean parameter varies inside the relative interior of $\m{P}$. Distributions with  mean value parameters lying well inside the cloud of points describing the support of the family have higher entropy, as their mass is distributed across a larger number of network configurations. In contrast, distributions with mean value parameters that are far removed from that cloud, including points very close to or on the boundary of $\m{P}$, have  lower entropy. It is worth pointing out that, for this specific family, the points in the support are closer to the lower boundary of the polygon $\m{P}$, while the side of $\m{P}$ determined by the convex hull of points corresponding to the empty and complete graph is significantly  distant from the support. This phenomenon becomes more pronounced as $g$ grows, so that this family will include many distributions, whose mean value parameters belong to a region far removed from the support, that would not provide a satisfactory or realistic explanation of any observed network, a feature that is often associated with degeneracy. 

%function does not exhibit fast changing behavior anywhere  inside the the convex support and has some global quality seems to exhibit a global behavunimodal behavior with no sudden 

\begin{figure}[t]
\centering
\includegraphics[width=4in,height=4in]{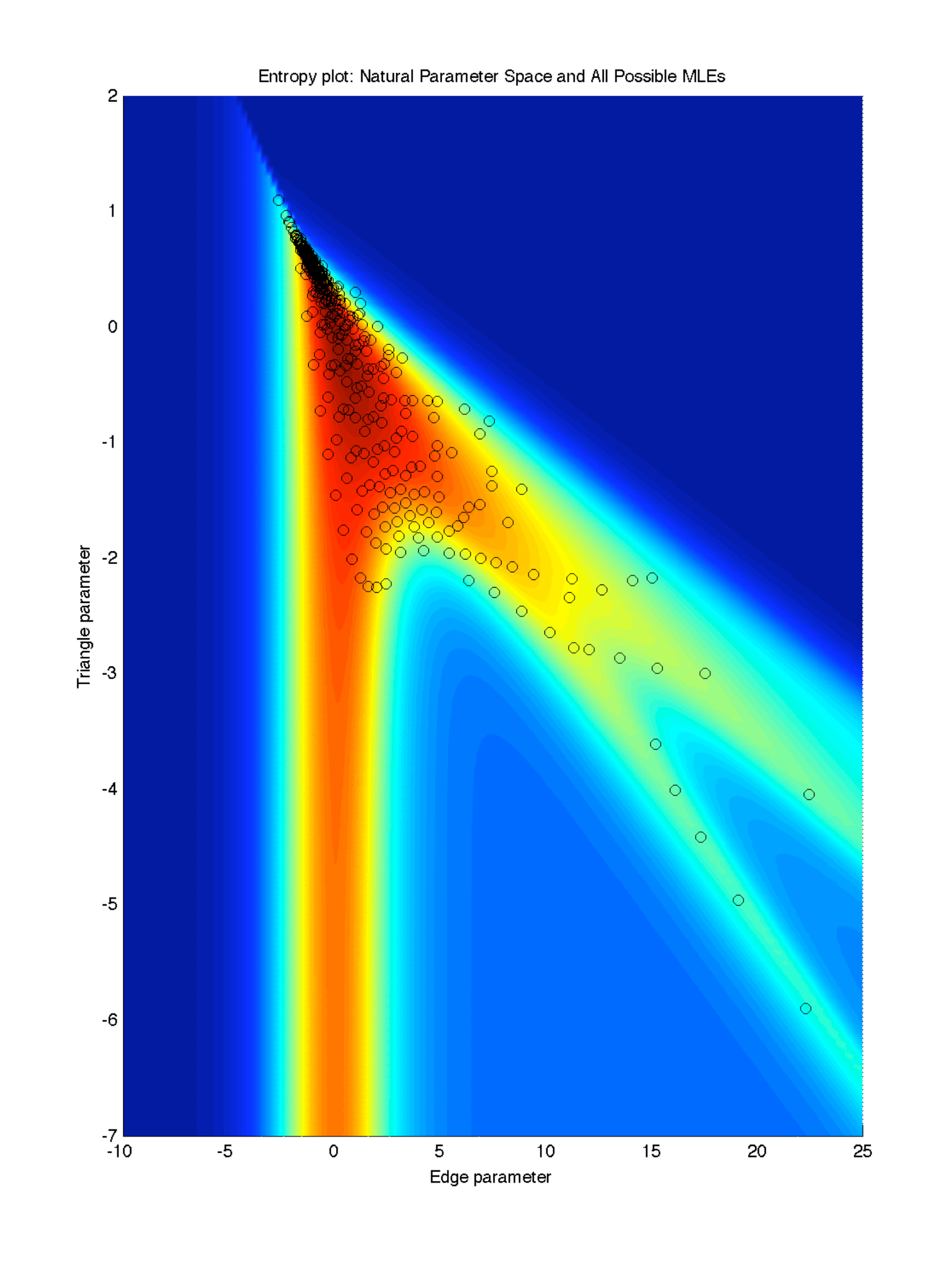}	
\caption{All possible MLEs of the natural parameters for the ERG model of Section \ref{sec:one} superimposed over the entropy plot of $S(\cdot)$.}
\label{fig:9.mle}
\end{figure}

In striking contrast, the natural parameter view of Figure \ref{fig:9.par} does not lend itself to immediate interpretations. In fact, although $S(\theta)$ and $V(\mu)$ are smooth functions related via the homeomorphism (\ref{eq:homeo}), $S(\theta)$ displays drastic localized behaviors, including multiple local maxima. 
%The entropy function,  a slowly-changing function over the mean value space, gets homeomorphically mapped by the inverse mean value parametrization into a strangely looking object in the natural parameter space. This object has some prominent features; 
In particular, the function $S(\theta)$  exhibits sharp changes and high-peaked ridges shooting at infinity along which it remains roughly constant. %among the most notable features of $S(\theta)$ from Figure \ref{fig:9.par} are the sharper changes and the high-peaked ridges shooting at infinity along which the entropy appears to remain constant. 
Furthermore, small variations in the natural parameter values cause big changes in the values of the entropy function, thus making this ERG model rather unstable, in the sense that neighboring parameters specify very different distributions, or at least distributions with different entropies. These features may in fact fall under the general umbrella of degeneracy, as described in Section  \ref{sec:degeneracy}. Finally, we remark that the portion of the natural parameter space containing parameter points that produce more realistic distributions with higher entropy values is relatively small, a characteristic that emerged from the inspection of Figure \ref{fig:9.mean} as well. In addition, the entropy function remains relatively high along some rays leaving the origin and shooting to infinity. We remark that Figure \ref{fig:9.par} matches quite closely analogous plots, not based on Shannon's entropy, for the same ERG model on graphs with 7 nodes by \cite{H:03}, although the interpretation of the plots using normal cones, as described  below, is missing.

Figure \ref{fig:9.mle} shows all the possible MLEs corresponding to the 415 points in the support of $\mathcal{E}_{\m{P}}$ that are inside $\m{P}$. These points are all the estimates that can be obtained by maximum likelihood procedure, so that, although the family $\mathcal{E}_{\m{P}}$ contains many other distributions, inference is only restricted to the 415 distributions identified by the MLEs, whose entropies are displayed in the Figure.

Part of the seemingly strange behavior of $S(\theta)$ can however be explained using the results derived in the previous section.  To that end, the convex support of Figure \ref{fig:9.supp}, can be expressed either as the convex hull of its vertices, namely %$\m{P}$ is the convex hull of the points 
\[
\m{P} =  \m{convhull} \left\{ (0,0), (20,0),(27,27),(30,44), (32,56), (36,84) \right\}
\]
or, equivalently, using the $\mathcal{H}$p-representation, as the solution set of a system of linear inequalities, i.e.  
\[
\m{P} = \{ t \in \mathbb{R}^2 \colon \m{A} t \leq b\},
\] 
where 
\[
\m{A} = \left[
\begin{array}{rcr}
0 & $\;$ & -1 \\
27 & $\;$& -7\\
17 & $\;$& -3\\
6 & $\;$& -1\\
7 & $\;$& -1\\
-21 & $\;$& 9\\
\end{array} \right], \quad \quad
b = \left[
\begin{array}{c}
0\\
540\\
432\\
136\\
168\\
0
\end{array} \right].
\]
The rows of $\m{A}$ identify the outer normals to the 6 sides of the polygon $\m{P}$ and generate the normal cones to the edges of $\m{P}$. The normal cone of a vertex of $\m{P}$ is the conic hull of the outer normals to the edges incident to that vertex. For example, the normal cone of the vertex $(0,0)$ is 
\[
 \m{cone}  \left\{ (0,-1), (-21,9) \right\}
\]
\begin{figure}[h!]
\centering
 \includegraphics[width=5in]{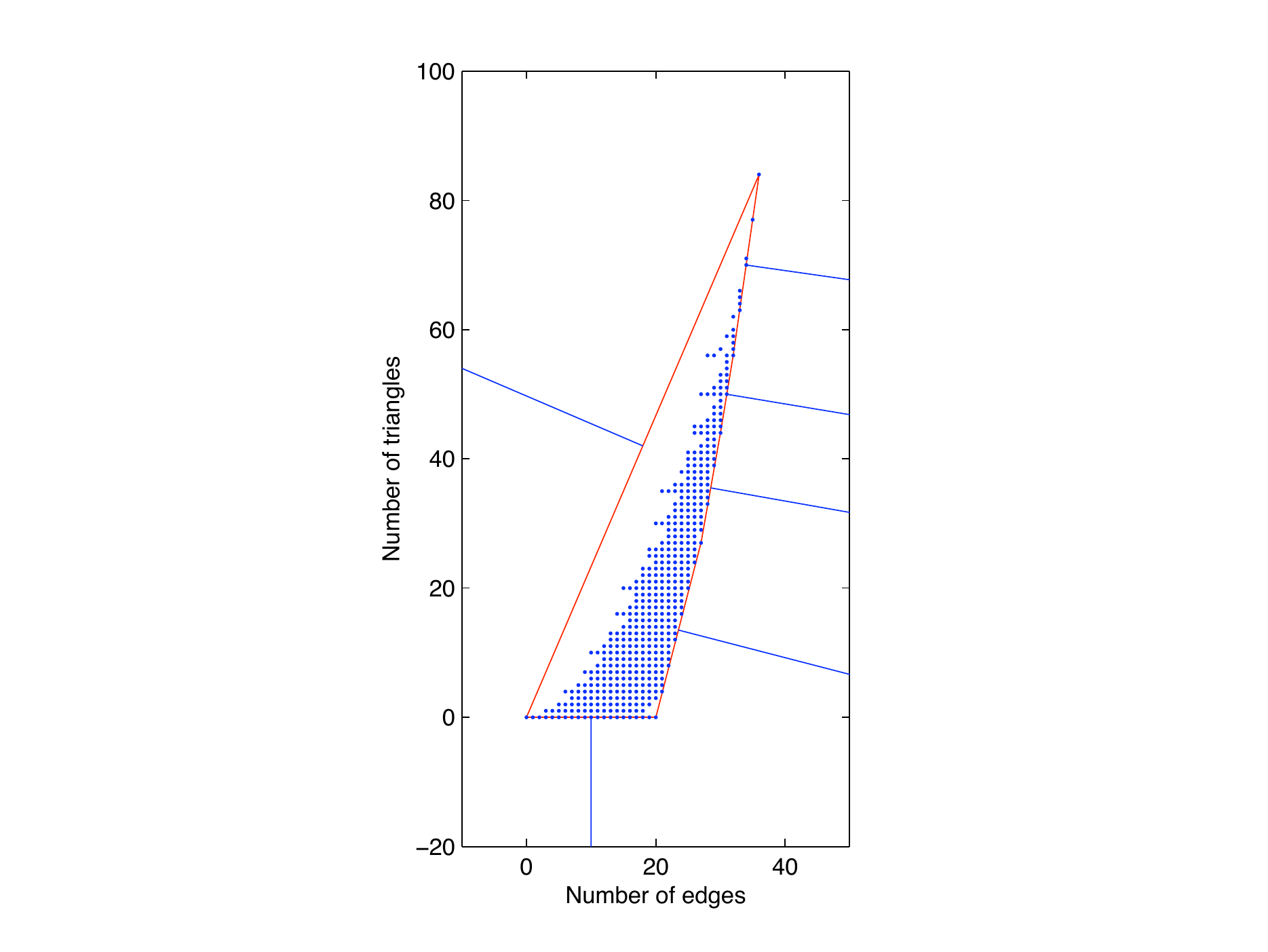}\\
\caption{Convex support and its outer normals for the ERG model of Section \ref{sec:one}.}
\label{fig:9.outer}
\end{figure}
The convex support $\m{P}$ and its outer normals  are shown in Figure \ref{fig:9.outer}.
It is immediate to picture that the normal fan of $\m{P}$, i.e. the collections of all the cones with apex at $0$ identified by the outer normals of $\m{P}$, partitions  $\mathbb{R}^2$.

\begin{figure}[h!]
\centering
 \includegraphics[width=4in]{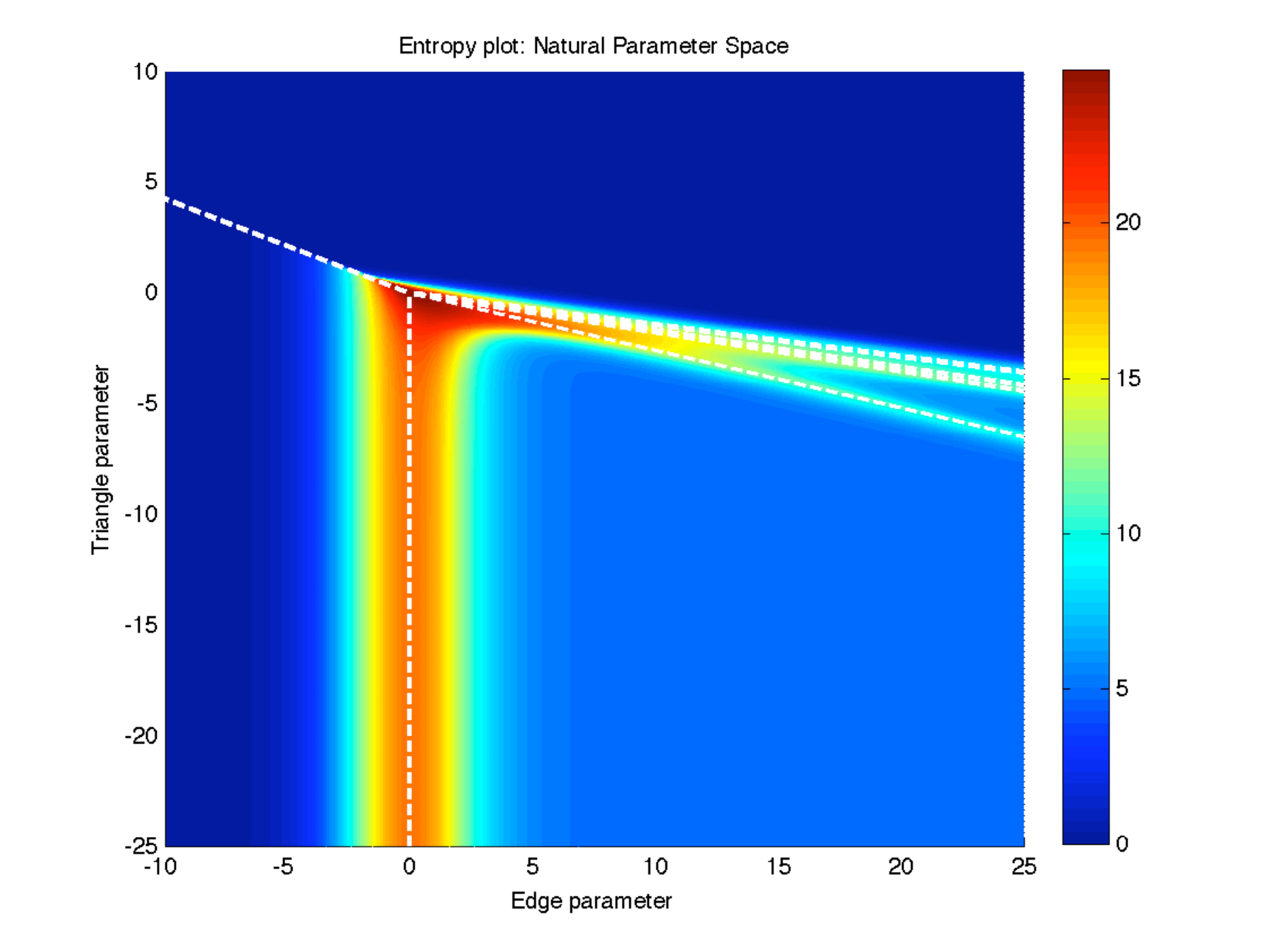}\\
\caption{Entropy plot of $S(\cdot)$ with, superimposed, the normal fan of $\m{P}$ for the ERG model of Section \ref{sec:one}.}
\label{fig:9.lines}
\end{figure}

Figure \ref{fig:9.lines} shows the  entropy plot over the subset $[10,25] \times [-25,10]$ of the natural parameter space with, superimposed, the normal fan of $\m{P}$, centered at the origin, which is the point of maximal entropy. As prescribed by Corollary \ref{cor:main}, the outer normal to $\m{P}$ are precisely the directions along which the closure of the original family $\mathcal{E}_{\m{P}}$ is realized, by adding the families $\mathcal{E}_F$, as $F$ ranges over the proper faces (in this case, edges and vertices) of $\m{P}$. These directions, starting at the origin, match perfectly the ridges of Figure \ref{fig:9.par}, along which the entropy function seems to converge to some fixed value. This is because any sequence $\{\theta_n\}$ along the outer normal of some edge $F$ will eventually no longer identifies distributions from the original family $\mathcal{E}_{\m{P}}$, but just \textit{one} distribution in $\mathcal{E}_F$ supported on $F$. Consequently, the entropy function does not change because, for all $n$ large enough, $\theta_n$ specifies almost the same distribution.

Figures \ref{fig:gui1}, \ref{fig:gui2} and \ref{fig:gui3} offer other two  pictorial representations of Corollary \ref{cor:main}. These plots were obtained using the MATLAB GUI available at \url{http://www.stat.cmu.edu/~arinaldo/ERG/} (see Section \ref{sec:soft} below). The left side of each plot shows the entropy function for the family of Section \ref{sec:one} along with the outer normals of $\m{P}$ leaving the original. The white circles represent  the selected natural parameter. The plots on the right show the support of the family. The red stars indicate the mean parameter values corresponding to the natural parameters indicated by the white circles on the left side of the figure. Points with darker shaded colors correspond to network statistics receiving high probability under the selected natural parameter.

Part {(\bf a)} of Figure \ref{fig:gui1} shows a distribution with high entropy, corresponding to a mean value parameter well inside the relative interior of $\m{P}$. In contrast, in parts {(\bf b)}, {(\bf c)} and {(\bf d)} the natural parameter is selected as $d$, with $d$ a point in the relative interior of the 2-dimensional normal cone of the vertex of coordinates $(0,0)$, which identifies the empty graph. Consequently, the entropy is almost $0$, as the associated distribution will put almost all its mass on that vertex of $\m{P}$. Notice that, even though the selected natural parameters from part {(\bf b)}, {(\bf c)} and {(\bf d)} are very different from each others, because they are far away from the set of parameters producing nondegenerate distributions and because they all to lye inside the normal cone of the vertex $(0,0)$, they parametrize essentially the same degenerate distribution on the  empty graph. 

Figure  \ref{fig:gui2} part {(\bf a)} shows the same phenomenon, but for the  different degenerate distribution putting virtually all its mass on the complete graph, which corresponds to the vertex $(36,84)$. As with Figure \ref{fig:gui1}, notice that the natural parameter is a point inside the normal cone of that vertex  and essentially any point in the upper triangular blue part of the entropy plot (which is, effectively, the relative interior of the associated normal cone) would parametrize this distribution. Part {(\bf b)} and {(\bf c)} show other degenerate distributions over the vertex of $\m{P}$ identified by points inside the interiors of the corresponding normal cones.
Figure \ref{fig:gui3} instead displays similar plots for a selection of natural parameters corresponding to directions lying on the normal cones, i.e. the outer normals, of some of the edges of $\m{P}$. 

\begin{figure}[!ht]
\centering
\begin{tabular}{c}
\includegraphics[width=4in]{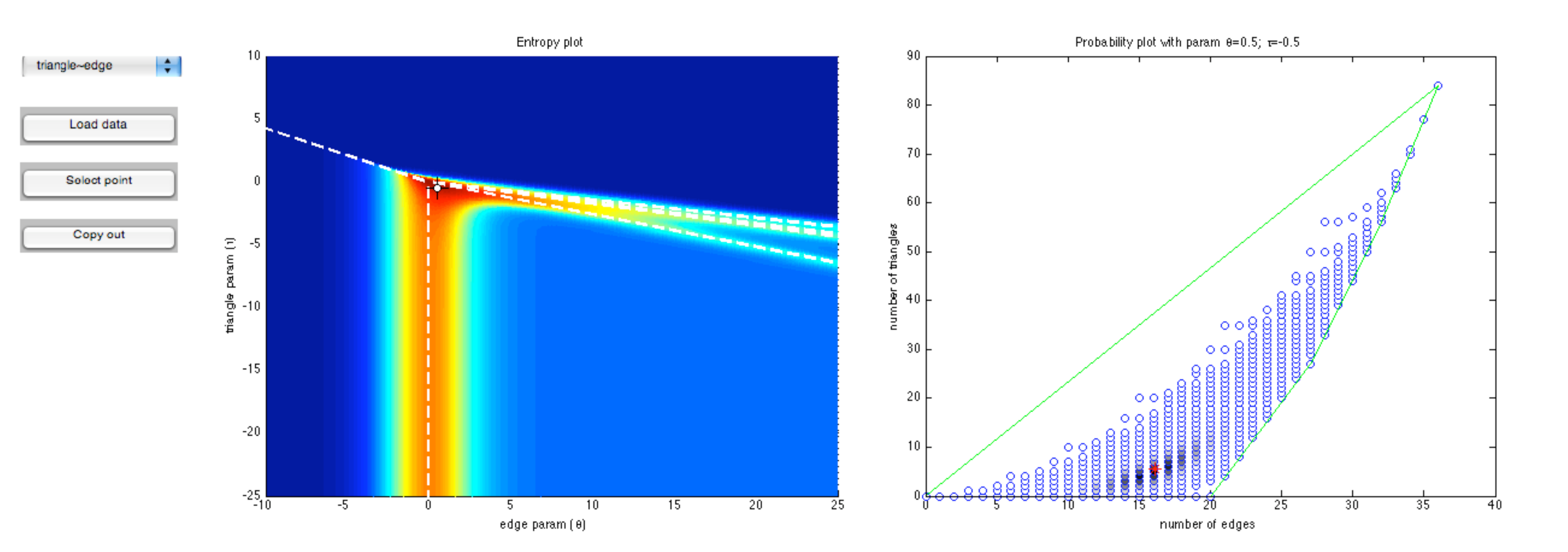}\\
(a)\\
\includegraphics[width=4in]{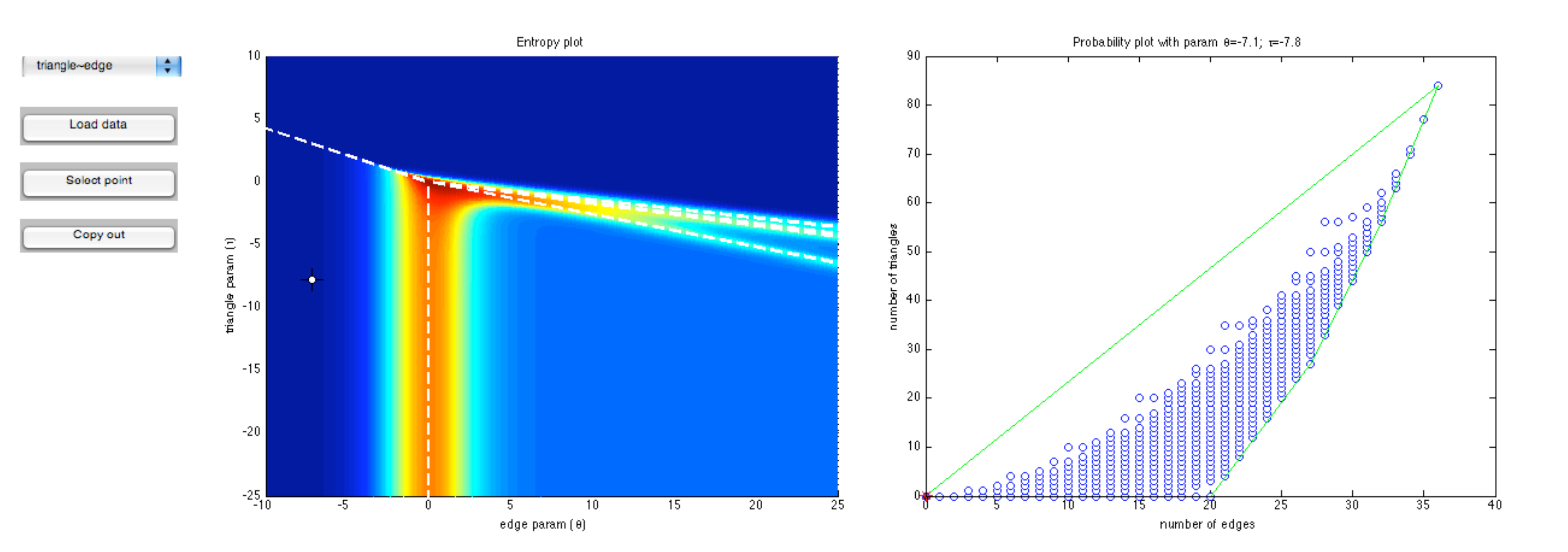}\\
(b)\\
\includegraphics[width=4in]{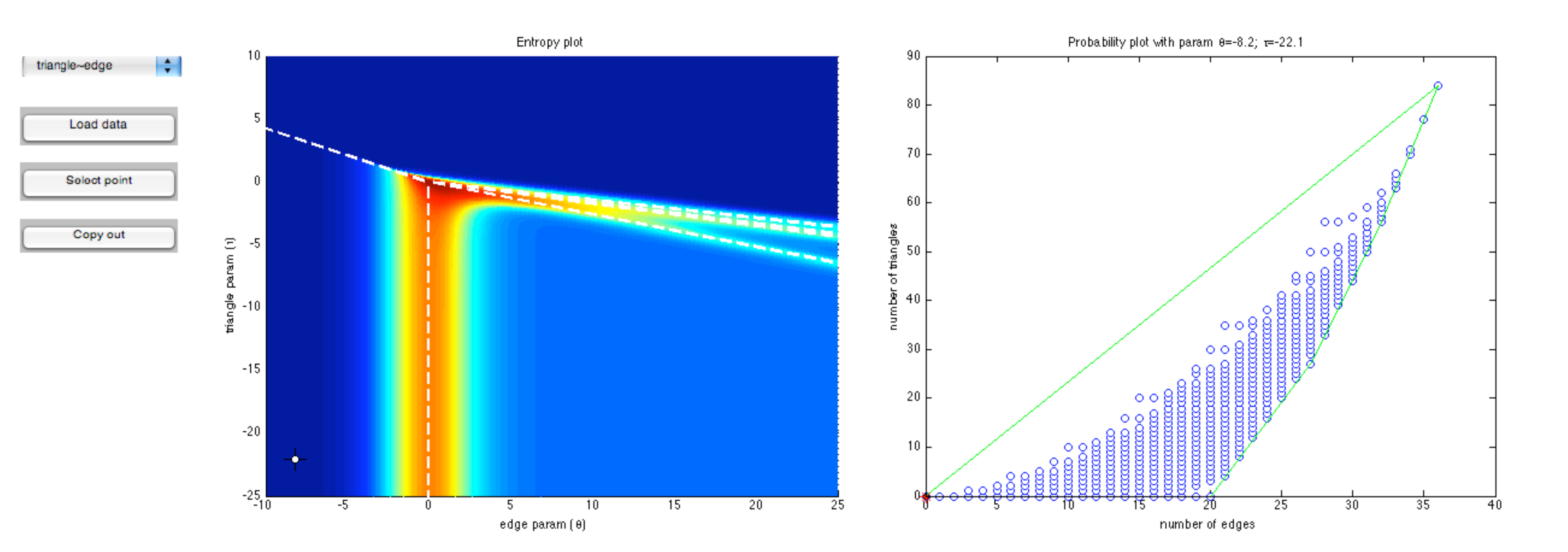}\\
(c)\\
\includegraphics[width=4in]{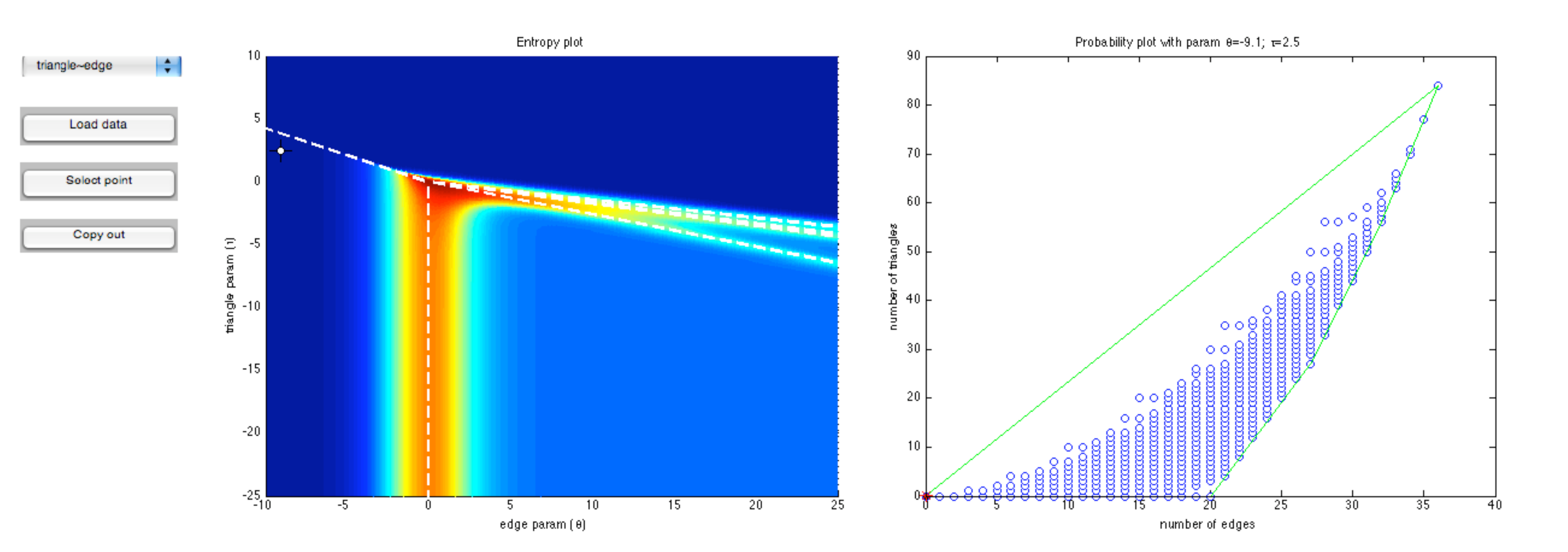}\\
(d)\\
\end{tabular}
\caption{Various distributions parametrized by points in the natural parameter space for the ERG model of Section \ref{sec:one}. The plots on the left are the entropy plots; the white points indicate the selected distributions. The plots on the right all display convex support. The red crosses represent the mean value parameters corresponding to the selected natural parameters, while the darker shading indicates network statistics configurations that are very probably under the selected parameters. Part {\bf (a)}: distribution with high-entropy with mean value parameter inside $\m{P}$. Parts {\bf (b)}, {\bf (c)} and {\bf (d)}: natural parameters all specifying distributions with virtually all of the total mass on the empty graph.}
\label{fig:gui1}
\end{figure}

\begin{figure}[!ht]
\centering
\begin{tabular}{c}
\includegraphics[width=5.2in]{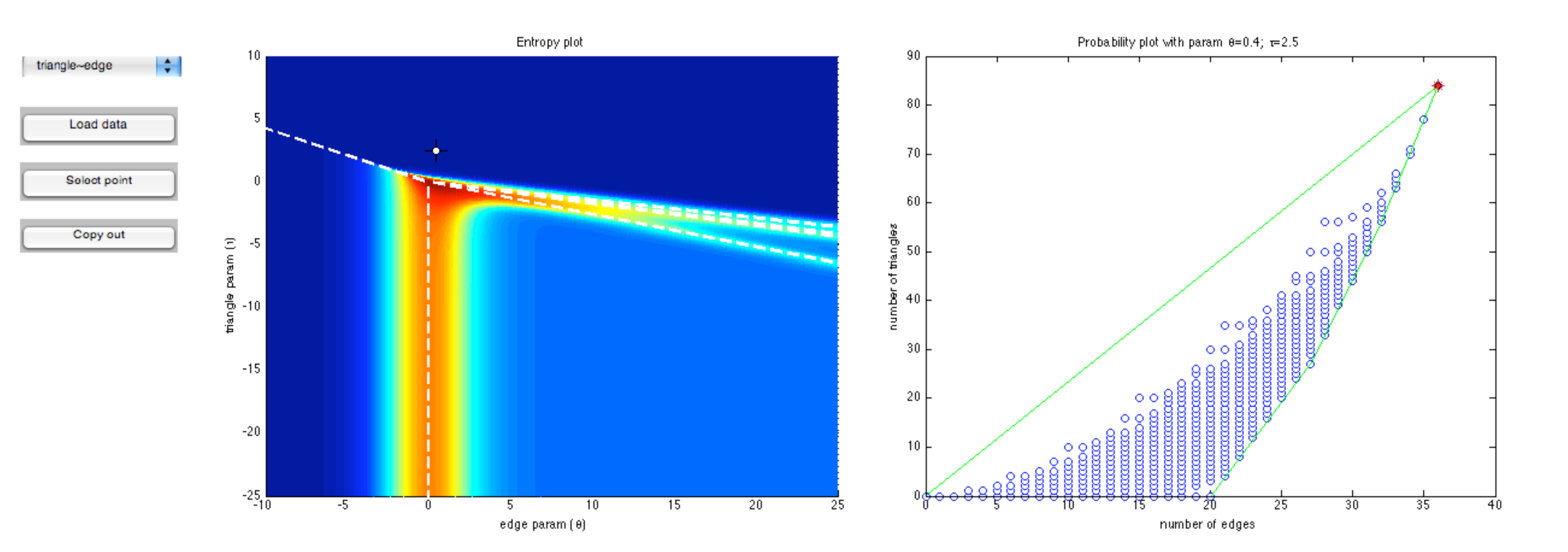}\\
(a)\\
\includegraphics[width=5.2in]{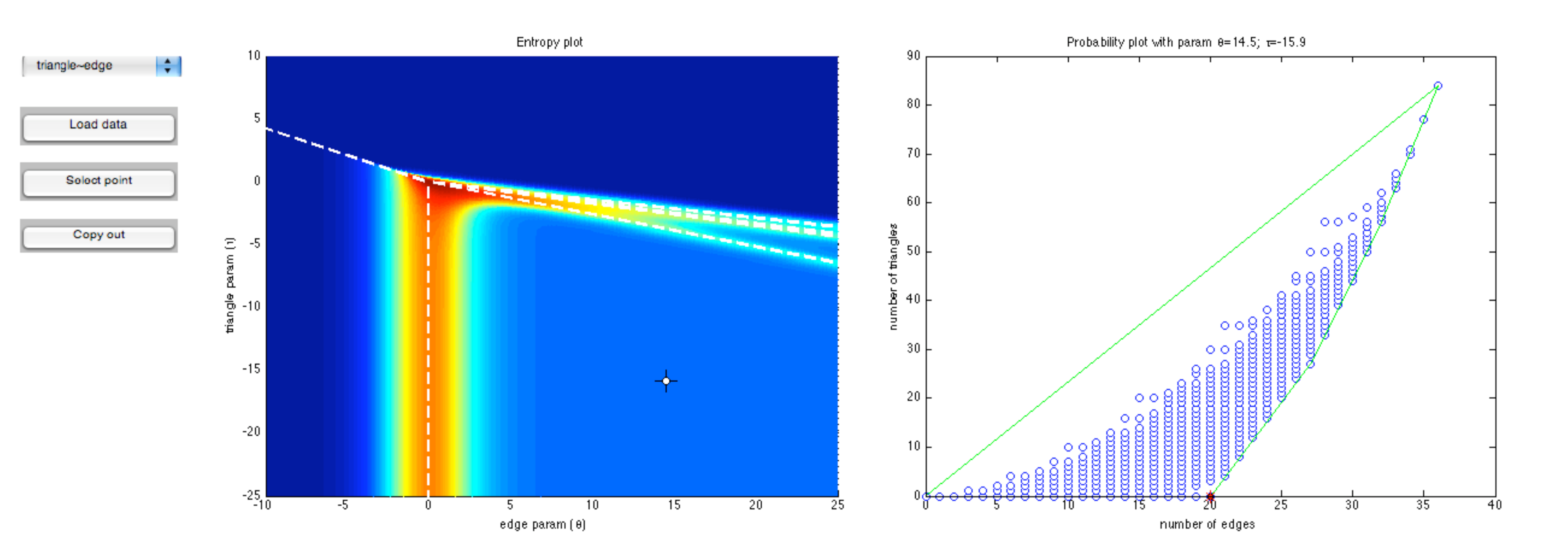}\\
(b)\\
\includegraphics[width=5.2in]{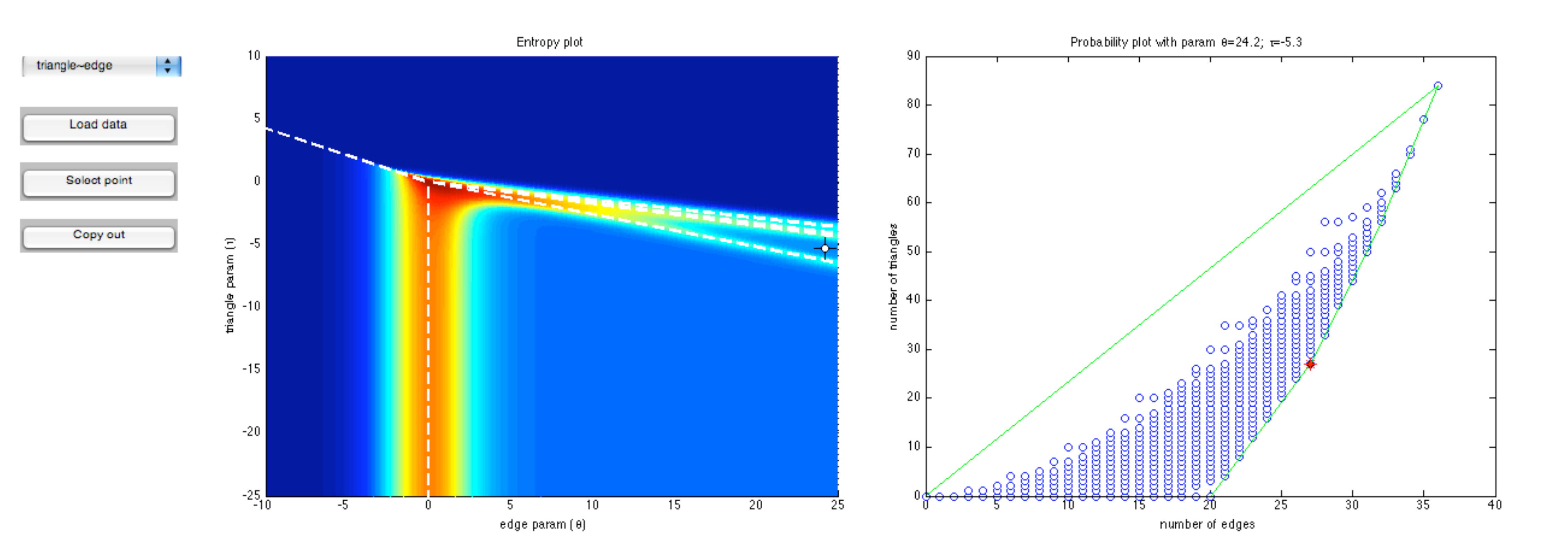}\\
(c)\\
\end{tabular}
\caption{Three degenerate distributions over three vertices of $\m{P}$. See the caption of Figure \ref{fig:gui1}.}
\label{fig:gui2}
\end{figure}

\begin{figure}[!ht]
\centering
\begin{tabular}{c}
\includegraphics[width=5.2in]{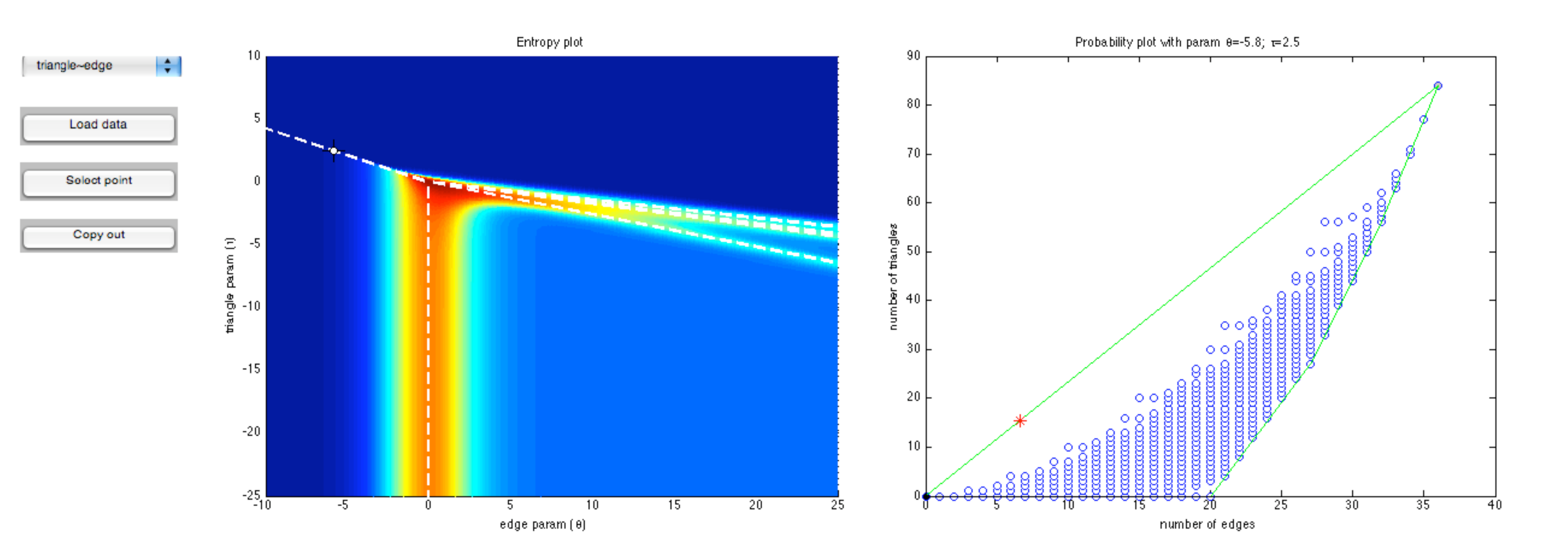}\\
(a)\\
\includegraphics[width=5.2in]{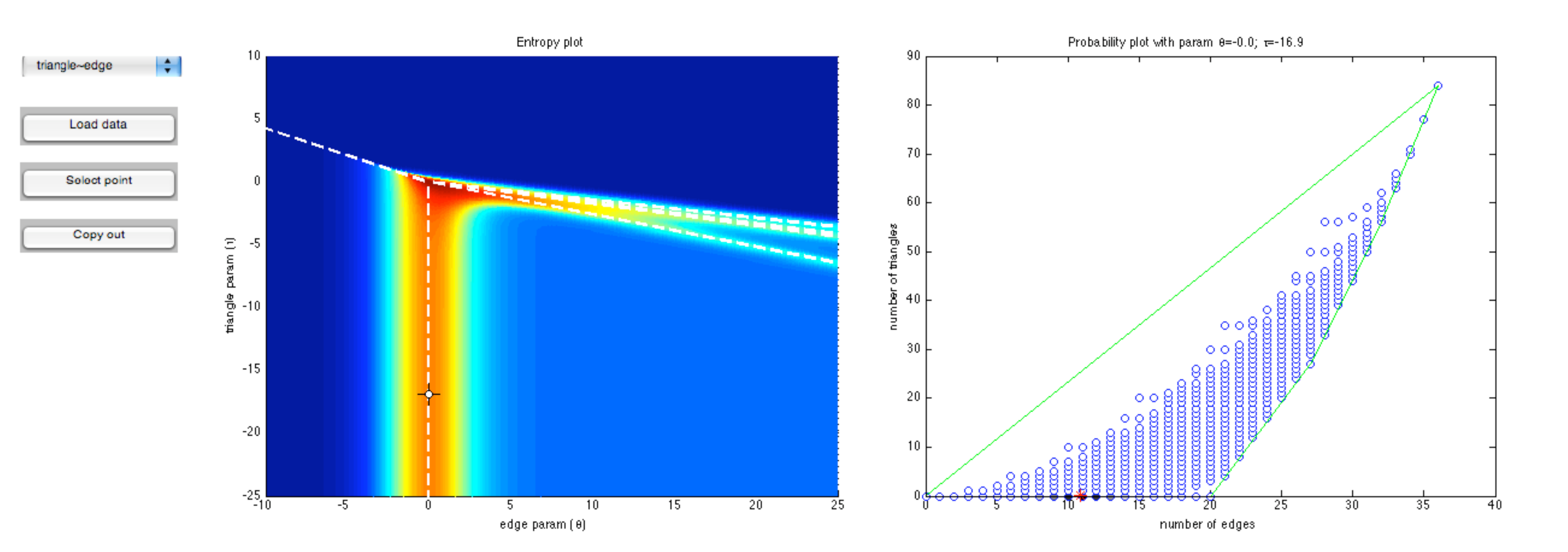}\\
(b)\\
%\includegraphics[width=4in]{gui_6.pdf}\\
%(c)\\
\includegraphics[width=5.2in]{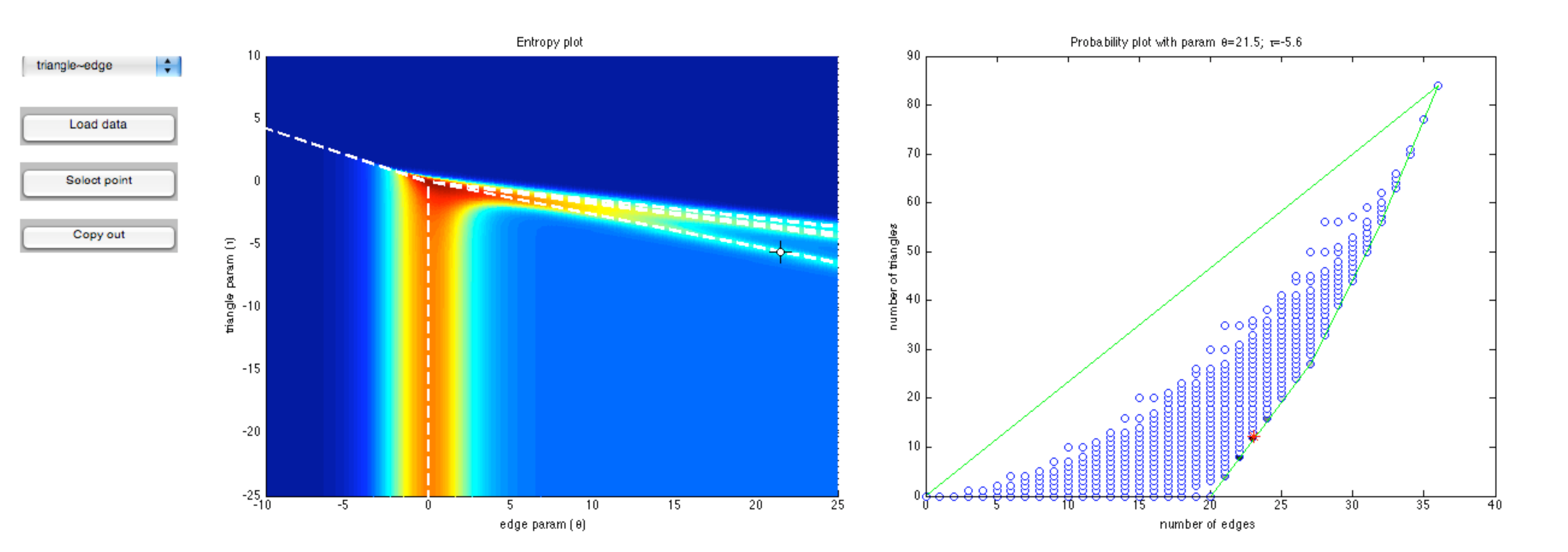}\\
(c)\\
\end{tabular}
\caption{Three degenerate distributions supported over three different edges of of $\m{P}$. See the caption of Figure \ref{fig:gui1}.}
\label{fig:gui3}
\end{figure}

\clearpage
\section{Discussion}

The purpose of this article has been two-fold. First, for the class of discrete linear exponential families with polyhedral convex support, we have characterized the extended family using the normal fan to the convex support. While complete results about closures of general exponential families exist in the literature, our restriction to families with polyhedral support  allowed us to obtain a more refined and explicitly geometrical description. In particular, our findings allowed us to gain a better understanding of the geometric and statistical properties of these families, as well as on the theoretical and algorithmic aspects of computing extended maximum likelihood estimates.  

Our second goal was to study the behavior and statistical properties of ERG models, that have seen widespread use for the statistical analysis of data for social networks. To that end, we applied the theoretical results derived in the first part of the article to one ERG model on the set of graphs with 9 nodes. Despite our analysis being mostly graphical (due to the lack of a closed-form expression for the log-partition function), it  captures a few interesting features of this model, some of which accounts for the seemingly strange behaviors that ERG models have been known to exhibit in practice, and generically termed degeneracy. Our investigation indicated that this type of behavior is, in fact, not unusual, and can be fully explained by the properties of linear discrete exponential families. Furthermore, based on similar experimentations with other ERG models, we believe our conclusions are not just specific to the model we present here but apply more widely to general ERG models.  

The application presented here are particularly relevant to ERG models built around network statistics that describes macroscopic features of the networks and whose dimension does not grow with the number of nodes. However, our results apply to more complex models, such as the original $p_1$ model of \cite{HL:81}, which has node-specific parameters and whose likelihood is based on an assumption of dyadic independence. For these models with many parameters, degeneracy is typically due to nonexistence of the MLE, which is very likely to occur if the network is even mildly sparse.
% i.e., the  presence or absence of edges between any pair of nodes, is independent of the presence or absence of edges between all other pairs.  The consequence of dyadic independence is a well-behaved likelihood surface.

Of course, much more needs to be done in order to fully understands the statistical subtleties, features and potential limitations of ERG models and in order to establish whether they are appropriate to model anything else than a large ensamble. Nonetheless, our contributions indicate that perhaps practitioners attribute to ERG models a degree of regularity that they may not possess.

\section{Acknowledgments}
The authors thank Mark Handcock and Surya Tokdar for helpful discussions on earlier drafts on this manuscript and Giovanni Leoni for illuminating clarifications.  This research was supported in part by NSF grant DMS-0631589 and a grant from the Pennsylvania Department of Health through the Commonwealth Universal Research Enhancement Program.

\newpage

\section{Appendix A: Proofs}
\begin{proof}[{\bf Proof of Lemma \ref{lem:nonid}.}]
Let $\theta_1 \in \Theta_F$ and $\zeta \in  \m{lin}(N_F)$ and consider the point $\theta_2 = \theta_1+\zeta$. We first show that $\theta_2 \in \Theta_F$ and $P_{\theta_1} = P_{\theta_2}$. Because $\zeta \in  \m{lin}(N_F)$, there exist some scalars $c_1\ldots,c_{m_F}$ such that
\[
\zeta = \sum_{i=1}^{m_F} c_i a_i.
\]
Therefore, almost everywhere $\nu_F$,
\[
\langle \zeta, x\rangle = \ \sum_{i=1}^{m_F} c_i b_i \equiv C.
\]
Then,
%\[
%\psi_F(\theta_1) = \log \int_{F} \exp \left\{ \langle \theta_2 + \zeta, x \rangle\right\} d \nu_F(x) = \log \int_{F} \exp \left\{ \langle \theta_2, x \rangle\right\} d \nu_F(x) + C = \psi_F(\theta_2) + C,
%\]
\[
\psi_F(\theta_2) = \log \int_{F} \exp^{ \langle \theta_1 + \zeta, x \rangle } d \nu_F(x) = \log \int_{F} \exp^{ \langle \theta_1, x \rangle } d \nu_F(x) + C = \psi_F(\theta_1) + C.
\]
As both $\psi_F(\theta_1)$ and $C$ are finite, it follows that $\psi_F(\theta_2) < \infty$ and, therefore, $\theta_2 \in \Theta_F$.
It is now easy to conclude that $P_{\theta_1} = P_{\theta_2}$ because, almost everywhere $\nu_F$,
\[
p_{\theta_2}(x) = \exp^{\langle \theta_1 + \zeta,x \rangle - \psi_F(\theta_2)} = \exp^{\langle \theta_1,x  \rangle + C - \psi_F(\theta_1) - C} = p_{\theta_1}(x).
\]
We now show that if $P_{\theta_1} = P_{\theta_2}$ and $\theta_1 \neq \theta_2$, then $ \theta_1 - \theta_2 \in \m{lin}(N_F)$. By Radon-Nykodin theorem, this occurs if and only if
\[
\langle x, \theta_1 - \theta_2 \rangle = \psi^F(\theta_1) - \psi^F(\theta_2) = D
\]
for some constant $D$, almost everywhere $\nu_F$.
As $\nu_F$ has support contained in $F$ and $F$ is defined by (\ref{eq:F}), the previous equality is equivalent to $\theta_1 - \theta_2 \in \m{lin}(N_F)$, thus completing the proof of the Lemma.

As for (\ref{eq:madai}), since  $\m{P}$ is full-dimensional and, almost everywhere $\nu_F$, $A_F x =b_F$, we have, for any $\theta \in \Theta_F$,
\[
0 = \m{Var}_\theta \left( \langle a, X \rangle \right) = a^\top I_F(\theta) a
\]
if and only if $a = \m{lin}(N_F)$.  This implies that $\m{rank}(I_F(\theta)) = \m{dim} \left( \m{lin}(N_F)^\bot \right) = \m{dim}(F)$.%, for all $\zeta \in \m{lin}(N_F)$.

\end{proof}

\begin{proof}[{\bf Proof of Lemma \ref{lem:a2}.}]
Arguing by contradiction, suppose that, for all $n$ large enough, $\theta_n$ belongs to a compact, hence bounded, set C. The facts that $\nabla \psi(\theta) = \mathbb{E}_\theta[X] \in \m{relint}(\m{P})$ $\Theta$, for each $\theta \in \Theta$ with finite norm, and that $\m{relint}(\m{P})$ and $\Theta$ are homeomorphic, imply that $\{ \nabla \psi(\theta), \colon \theta \in C \}$ is a compact subset of $\mathrm{relint}(\m{P})$. Then, because $\| \nabla \psi(\theta) - \mu_F \|_2$ is a continuous function of $\theta$, for all $\theta \in \Theta$,  $\inf_{\theta_n \in C} \| \nabla \psi(\theta_n) - \mu_F \|_2 = \| \nabla \psi(\theta^*) - \mu_F \|_2$ for some $\theta^* \in C$. But then, $\nabla \psi(\theta^*) \equiv \mu^* \in \mathrm{relint}(\m{P})$ so that, $\| \mu^* - \mu_F \|_2 > \epsilon > 0$ for some $\epsilon$, which produces a contradiction. %Therefore, for each $n$, we can write $\theta_n = \theta + \rho_n d_n$, for some positive numbers $\rho_n$ such that $\rho_n \rightarrow \infty$. %We claim that, for all $n$ large enough, \ale{here!!!} $d_n \in N_F$. 
\end{proof}

\begin{proof}[{\bf Proof of Theorem \ref{thm:main}.}]
Throughout the proof, we will write $S_{k-1} = \{ x \in \mathbb{R}^k \colon \|x \|_2 = 1\}$.

In the proof we will make use repeatedly of the following decomposition. For any point $x_0 \in \m{P}$ and proper face $F$ of $\m{P}$, we will write
\begin{equation}\label{eq:dec}
p_{\theta_n}(x_0) = \frac{\exp^{ \langle \eta, x_0 \rangle }}{A_{0,n}(x_0,F) + A_{>,n}(x_0,F) + A_{<,n}(x_0,F)},
\end{equation}
 where
\[
A_{0,n}(x_0,F) = \int_{\{ x \colon  A_F (x - x_0) = 0\}} \exp^{ \langle \eta, x \rangle + \rho_n \langle d_n, x - x_0 \rangle} d \nu(x),
\]
\[
A_{>,n}(x_0,F) = \int_{\{ x \colon  A_F (x - x_0) > 0\}} \exp^{ \langle \eta, x \rangle + \rho_n \langle d_n, x - x_0 \rangle} d \nu(x),
\]
and
\[
A_{<,n}(x_0,F) = \int_{\{ x \colon  A_F (x - x_0) < 0\}} \exp^{ \langle \eta, x \rangle + \rho_n \langle d_n, x - x_0 \rangle} d \nu(x).
\]
Notice that, for all $n$, if $x_0 \in F$, then $A_{>,n}(x_0,F) = 0$, since $\nu\{ x \colon  A_F (x - x_0) > 0 \} = 0$.
We will also use the following fact, which stems directly from Lemma \ref{lem:nonid}: $\exp^{\langle \eta, x\rangle - \psi^F(\eta)} = \exp^{\langle \theta, x\rangle - \psi^F(\theta)} = p_{\theta^F}^F(x)$, almost everywhere $\nu_F$.
\begin{enumerate}
\item Party 1.\\ 
We will begin by showing sufficiency.
First, we consider the case of a generic point $x_0 \in F$. 
If,  $d_n \in \m{ri}(N_F)$, then, by part {\it 1.} of Lemma \ref{lem:nf}, $ \langle d_n, x - x_0 \rangle = 0$ for all $x \in F$, which implies that
\[
A_{0,n}(x_0,F) = \int_F \exp^{ \langle \eta, x \rangle}d \nu(x) = \exp^{\psi^F(\eta)},
\]
for al $n$.
On the other hand, for any $x \not \in F$, since $R$ is a  compact subset of $\m{ri}(N_F) \cap S_{k-1}$ and $\{ d_n \} \in R$, we have
\[
\sup_n \langle d_n, x - x_0 \rangle \leq \sup_{d \in R} \langle d, x - x_0 \rangle = \langle d_x^*, x - x_0 \rangle,
\]
for some $d_x^* \in R$, which may depend on $x$. Furthermore, by part {\it 2.} of Lemma \ref{lem:nf}, $ \langle d_x^*, x - x_0 \rangle  < 0$. Thus,  $ \rho_n \langle d_n, x - x_0 \rangle  \rightarrow -\infty$, for each $x \not \in F$. for each  $x \in \{ x \colon  A_F (x - x_0) < 0\}$,
\[
\exp^{ \langle \eta, x \rangle + \rho_n \langle d_n, x - x_0 \rangle} \leq \exp^{ \langle \eta, x \rangle},
\]
whereby 
\[
\int_{\{ x \colon  A_F (x - x_0) < 0\}} \exp^{ \langle \eta, x \rangle}d \nu(x) \leq  \int_{\mathbb{R}^k} \exp^{ \langle \eta, x \rangle}d \nu(x) =  \exp^{ \psi(\eta) } < \infty.
\]
Then, by the dominated convergence theorem, we obtain
\[
A_{<,n}(x_0,F) \searrow 0.
\]
Therefore,
\[
A_{0,n}(x_0,F)  + A_{<,n}(x_0,F) \searrow \exp^{\psi^F(\eta)},
\]
which implies that
\begin{equation}\label{eq:up}
\lim p_{\theta_n}(x_0) \nearrow \exp^{ \langle \eta, x_0 \rangle  - \psi^F(\eta)} = p_{\theta_F}(x_0).
\end{equation}
%where the last identify stems form Lemma \ref{lem:nonid}. v

Next, let $x_0 \in \m{P} \cap F^c$ and notice that
\[
A_{>,0}(x_0,F) + A_{0,n}(x_0,F)  + A_{<,n}(x_0,F) \geq A_{>,n}(x_0,F) \geq \int_F  \exp^{ \langle \eta, x \rangle + \rho_n \langle d_n, x - x_0 \rangle} d \nu(x),
\]
since $F \subseteq \{ x \colon A_F(x - x_0) > 0 \}$.
For any $x \in F$, since $\{ d_n \} \in R$ and $R$ is a  compact subset of $\m{ri}(N_F) \cap S_{k-1}$, we get
\[
\inf_n \langle d_n, x - x_0 \rangle \geq \inf_{d \in R} \langle d, x - x_0 \rangle = \langle d_x^*, x - x_0 \rangle,
\]
for some $d^*_x \in R$, which may depend on $x$.
By   Lemma \ref{lem:nf}, part {\it 2.}, $\rho_n \langle d^*_x, x - x_0 \rangle \rightarrow \infty$, for all $x \in F$. But then, as  $\nu(F) > 0$ by assumption (A3), we obtain
\[
\int_F  \exp^{ \langle \eta, x \rangle + \rho_n \langle d_n, x - x_0 \rangle} d \nu(x) \rightarrow \infty,
\]
by the monotone convergence theorem. Thus,
\begin{equation}\label{eq:mon}
A_{>,0}(x_0,F) + A_{0,n}(x_0,F)  + A_{<,n}(x_0,F) \rightarrow \infty,
\end{equation}
and, therefore, $p_{\theta_n}(x_0) \rightarrow 0 = p^F_{\theta_F}(x_0)$.

\item Part 2.\\
Suppose that, $\{ d_n \} \subset R$, where $R$ is a  compact subset of $ N^c_F$. 
Then, there exists a subsequence $\{ d_{n_k}\} \subset \{ d_n \}$ such that, for all $k$ large enough, $d_{n_k}$ belongs to a  compact set $R^*$ such that either $R^* \subset \m{ri}(N_{F'})$, for some $F' \neq F$, or $R^* \subset \left( \mathcal{N}(\m{P}) \right)^c$. In the latter case,  by part 3., proven below, the numbers $\| \mu_{n_k} - \mu^F\|_2$  grow unbounded and, therefore, (\ref{eq:thm:main}) is violated. 

In the former case, by part 1. of the proof, (\ref{eq:thm:main}) is verified for $F'$, so it cannot be simultaneously verified for $F$ as well. Indeed,  $p_{\theta_n}$ cannot converge pointwise to both $p_{\theta^F}^F$ and $p_{\theta^{F'}}^{F'}$, which identify different probability distributions with different supports.

\item Part 3.\\
We will show that, if $\{ d_n \} \subset R$ for some  compact subset of $ \left( \mathcal{N}(\m{P}) \right)^c$, then, 
\begin{equation}\label{eq:remark}
p_{\theta_n}(x_0) \rightarrow 0, \quad \forall x_0 \in \m{P}.
\end{equation}
This implies that $\| \mu_n \|_2 \rightarrow \infty$.
Let $x_0 \in \m{P}$. As $\m{P}$ is full-dimensional and $d_n \not \in \mathcal{N}(\m{P})$, by Lemma \ref{lem:nf}, part {\it 3.}, the set $S_n = \{ x \in \m{P} \colon \langle d_n, x - x_0\rangle > 0 \}$ is non-empty, for each $n$. Furthermore, since, by assumption, 
\[
\inf_{d \in R} \inf_{d' \in \mathcal{N}(\m{P})} \| d - d' \|_ 2 > 0,
\]
the set $S  = \m{lim \ inf}_n S_n$ is non-empty as well. We now claim that $\nu(S) > 0$. In fact, arguing by contradiction, suppose that $\nu(S) = 0$. Then, there exists a subsequence $\{ d_{n_k} \} \subset \{ d_n \}$ such that no point $x \in \m{supp}(\nu)$ can satisfy $\lim_k \langle d_{n_k}, x - x_0 \rangle > 0$. However, since, by assumption (A3), $\m{P} = \m{convhull}(\m{supp}(\nu))$, this implies that $\m{P} \subseteq \{ y \colon \lim_k \langle d_{n_k}, y - x_0\rangle \leq 0 \}$, which in turn implies that $\lim_k d_{n_k} \in \mathcal{N}(\m{P})$, violating the condition that $\{ d_n \}$ is bounded away from $\mathcal{N}(\m{P})$. Thus, $\nu(S) > 0$, from which we can conclude that $\m{lim \ inf}_n \nu(S_n)  \geq \nu(S) > 0$.
Then, by the monotone convergence theorem,
%We claim that $\inf_n \nu(S_n) > 0$. Indeed, by assumption A3, $\m{P} = \m{convhull}(\m{supp}(\nu))$, so that, if $\nu(S_n) = 0$,  no point $x$ in $\m{supp}(\nu)$ can satisfy $\langle d_n, x - x_0\rangle > 0$. This implies that $\m{P} \subseteq \{ y \colon \langle d_n, y - x_0\rangle \leq 0 \}$, which in turn implies that $d_n$ is the normal vector to a supporting hyperplane of $\m{P}$ and, hence that $d_n \in \mathcal{N}(\m{P})$, which is a contradiction. Furthermore, since $R$ is relatively compact, Thus, $\nu(S) > 0$. Then, by the monotone convergence theorem,
\[
\int_{S_n} \exp^{ \langle \eta, x \rangle + \rho_n \langle d_n, x - x_0 \rangle} d \nu(x) \rightarrow \infty,
\]
Therefore, 
\[
p_{\theta_n}(x_0) = \frac{\exp^{ \langle \eta, x \rangle}}{\int_{S_n} \exp^{ \langle \eta, x \rangle + \rho_n \langle d_n, x - x_0 \rangle} d \nu(x) + \int_{S_n^c} \exp^{ \langle \eta, x \rangle + \rho_n \langle d_n, x - x_0 \rangle} d \nu(x)} \rightarrow 0,
\]
as claimed. 
\end{enumerate}
\end{proof}

%\begin{proof}[{\bf Proof of Corollary \ref{cor:uno}.}]
%The first part of the claim follows straightforwardly from Theorem \ref{thm:main}, while the second part is proved by noticing that, under the assumptions of the Corollary, convergence is monotone in equation (\ref{eq:up}). 
%\end{proof}

%\begin{proof}[{\bf Proof of Corollary \ref{cor:due}.}]
%Because $ \m{C} = \{ 0 \} $ and $\m{P}$ is assumed to be full-dimensional, the support of $\mathcal{N}(\m{P})$ is $\mathbb{R}^k$, thus any non-zero vector belongs to some $N_F$ and is, therefore a direction of recession. Thus, if $\| \theta_n \|_2 \rightarrow \infty$, it is possible to extract a subsequence $\{ \theta_{n_k} \}$ satisfying the claim, for some $F$ and any $\theta \in \Theta$. The result then follows from Theorem \ref{thm:main}.
%\end{proof}

\begin{proof}[{\bf Proof of Corollary \ref{cor:main}.}]
Any direction $d \in \mathbb{R}^k$ is either in $\mathcal{N}(\m{P})$, in which case, it must belong to $\m{ri}(N_F)$ for one face $F$ of $\m{P}$ or  in $\left( \mathcal{N}(\m{P}) \right)^c$. The results then follow directly from Theorem \ref{thm:main}.
\end{proof}

\begin{proof}[{\bf Proof of Corollary \ref{cor:dr}.}]

If $x \in \m{ri}(\m{P})$, then the MLE exists, is unique and is given by the vector $\widehat{\theta} \in \Theta$ such that $\nabla \psi(\widehat{\theta}) = x$. Equivalently,  since in this case $N_{\m{P}}  = \{ 0 \}$, invoking Corollary \ref{cor:main}, part {\it 1.}, $-\ell_x$ has no direction of recession.
Thus consider the case of $x \in \m{rb}(\m{P})$ and let $F$ be the unique face such that $x \in \m{ri}(F)$.
If $d \in \m{ri}(N_F)$, then by Corollary \ref{cor:main}  part {\it 1.},
\begin{equation}\label{eq:zero}
\lim_{\rho \rightarrow \infty}p_{\theta + \rho d}(x) > 0,
\end{equation}
so (\ref{eq:dir}) holds. Suppose now that $d \in \m{rb}(N_F)$. %We will show that, (\ref{eq:zero}) is verified, which implies that  $d$ is a direction of recession.
Let  $\mathcal{F}_x = \{ F' \colon x \in \m{rb}(F') \}$, with $F'$ being a face of $\m{P}$. By Lemma \ref{lem:nf}, part {\it 4.},
\[
\biguplus_{F' \in \mathcal{F}_x} \m{ri}(N_{F'}) = \m{rb}(N_F),
\]
so, if $d \in \m{rb}(N_F)$, then $d \in \m{ri}(N_{F'})$, for some $F' \in \mathcal{F}_x$. By Corollary \ref{cor:main}, part {\it 1.}, almost everywhere $\nu_{F'}$,
\[
\lim_{\rho \rightarrow \infty}p_{\theta + \rho d} = p^{F'}_{\theta^{F'}}.
\]
Since $x \in F'$, we have $\nu_{F'}(x) > 0$, which implies, $p^{F'}_{\theta^{F'}}(x) > 0$ and, consequently, (\ref{eq:zero}). Thus $d$ is also a direction of recession and we have shown that any point in $N_F$ is a direction of recession for $-\ell_x$,

It remains to be shown that Equation (\ref{eq:dir}) is not verified if $d \not \in N_F$. If $d \not \in \mathcal{N}(\m{P})$ Corollary \ref{cor:main}  part {\it 2.} yields
\[
\lim_{\rho \rightarrow \infty}p_{\theta + \rho d}(x) = 0,
\]
hence $-\ell_x(\theta) \rightarrow \infty$, so $d$ is not a direction of recession. If instead $d \in \mathcal{N}(\m{P}) \cap N^c_F$, then it must be the case that $d\in \m{ri}(N_F^*)$, for some face $F^*$ such that $F \cap F^* = \emptyset$, otherwise $N_{F^*} \subset \m{rb}(N_F)$ (see, e.g., Lemma \ref{lem:nf}, part {\it 4.}). Thus, by Corollary \ref{cor:main}, part {\it 1.},
 \[
 \lim_{\rho \rightarrow \infty}p_{\theta + \rho d}(x) = p^{F^*}_{\theta^{F^*}}(x) = 0,
 \]
 because $x \not \in F^*$, while $p^{F^*}_{\theta^{F^*}}(x) > 0$ only if $x \in F^*$.
As a result, (\ref{eq:zero}) does not hold, so that $d$ does not satisfy (\ref{eq:dir}) and is not a direction of recession.

\end{proof}

\begin{proof}[{\bf Proof of Corollary \ref{cor:uno}.}]
%As  $\mathbb{E}_{\widehat{\theta}_F}[X] = x$, we get that $x \in \m{relint}(F)$. 
The only interesting case is when $x \in \m{ri}(F)$, for some proper face $F$, otherwise $\mathcal{N}(\m{P}) = \{ 0 \}$, and $-\ell_x$ has no directions of recession, as the MLE exists.
For every $\theta \in \Theta$, let $\{ \theta_n \}$ be a $(\theta,\{ \rho_n\},d)$-sequence. By Corollary \ref{cor:dr}, we need to consider only the case $d \in N_F$.
If $d \in \m{ri}(N_F)$, by Lemma \ref{cor:up}, $p_{\theta_n}(x) \nearrow p^F_{\theta_F}(x)$. Now suppose that $d \in \m{rb}(N_F)$. Then, $d \in \m{ri}(N_{F^*})$ for some face $F^*$ such that $F \subset F^*$. Another application of Lemma \ref{cor:up}, yields $p_{\theta_n}(x) \nearrow p^{F^*}_{\theta_{F^*}}(x)$. However,
\[
p^{F^*}_{\theta_{F^*}}(x)  = \exp^{\langle \theta,x \rangle - \psi^{F^*}(\theta)} < \exp^{\langle \theta, x \rangle - \psi^{F}(\theta)} = p^F_{\theta_F}(x),
\]
since
\[
\exp^{\psi^{F^*}(\theta) } = \int_{F^*} \exp^{\langle \theta,z \rangle} d\nu(z) \geq \int_{F} \exp^{\langle \theta,z \rangle} d\nu(z)= \exp^{\psi^{F}(\theta) }.
\]
Thus, $\sup_{\theta \in \Theta}p_\theta(x) =  p^F_{\gamma_F}(x)$ for some $\gamma^F \in \Theta_F$. But $\sup_{\gamma_F \in \Theta_F} p_{\gamma_F}(x) = p^F_{\widehat{\theta}_F}(x)$, since only the points $\theta \in \widehat{\theta}_F$ satisfy the first order optimality conditions (\ref{eq:mle.F}). The result follows.
\end{proof}

\begin{proof}[{\bf Proof of Corollary \ref{cor:tre}}]
Part \textit{i)} and \textit{ii)} follows from Lemma \ref{lem:nonid} and results of \cite{CS:03,CS:05}. Part \textit{iii)} is a direct consequence of part \textit{i)}.
\end{proof}

\begin{proof}[{\bf Proof of Corollary \ref{cor:due}}]
%Part \textit{i)} and \textit{ii)} follows from Lemma \ref{lem:nonid} and results of CSISZAR AND MATUS. Part \textit{iii)} is a direct consequence of part \textit{i)}. As for part \textit{iv)}, 
%Since $\mathcal{E}_F$ is degenerate, $\Lambda_{\min}(I_F(\theta)) = 0$ for every $\theta \in \Theta$. To see this, let $\theta' \in \theta_F$. Then, since $\langle \theta - \theta',x \rangle$ is a constant, almost everywhere $\nu_F$, $0 = \mathbb{V}_{\theta}(\langle \theta - \theta',X \rangle) = (\theta - \theta')I_F(\theta)(\theta - \theta')$, thereby $\Lambda_{\min}(I_F(\theta)) = 0$.
%Notice also that $\Lambda_{\min}(I(\theta))>0$ for all $\theta \in \Theta$. The result follows from part 1. of Theorem \ref{thm:main} and from the analytic properties of the exponential families. \ale{ok, ok, more details....}
For any $\theta \in \Theta$, the $(i,j)$-th entry of $I(\theta)$ is \citep[see, e.g., Corollary 2.3 in][]{BRW:86}
\[
I_{i,j}(\theta) = \frac{\partial}{\partial \theta_i \partial \theta_j} \psi(\theta).
\]
From the proof of Theorem \ref{thm:main}, $\psi(\theta_n) \rightarrow \psi^F(\theta + \zeta)$, for every $\zeta \in \m{lin}(N_F)$. Then, by the analytic properties of the cumulant generating function \citep[see, e.g.][Chapter 2]{BRW:86}, we obtain
\[
\lim_n \frac{\partial}{\partial \theta_i \partial \theta_j} \psi(\theta_n) = \frac{\partial}{\partial \theta_i \partial \theta_j} \lim_n \psi(\theta_n) = \frac{\partial}{\partial \theta_i \partial \theta_j} \psi^F(\theta+ \zeta) = I_F(\theta + \zeta),
\]
for every $\zeta \in \m{lin}(N_F)$, hence the statement is proved.
\end{proof}

\section{Appendix B}
%We being by noting that Equation (\ref{eq:up}) in the proof of Theorem \ref{thm:main} reveals that, if we consider only points in $\m{supp}(\nu_F)$, the pointwise convergence of the densities occur from below, as indicated in the next Corollary.
The following lemma in needed in the proof of Corollary \ref{cor:uno}

\begin{lemma}\label{cor:up}
Under the conditions of Corollary \ref{cor:main}, $ p_{\theta_n} \nearrow  p^F_{\theta_F}$, a.e. $\nu_F$, if and only if $d \in \m{relint}(N_F)$.
\end{lemma}

\begin{proof}
The claim follows from Equation (\ref{eq:up}) in the proof of Theorem \ref{thm:main}, which holds for all $x \in F$, thus almost everywhere $\nu_F$.
\end{proof}

Below, we collect some basic facts about the normal fan and normal cones needed in our proofs.
With some slight abuse of notation, we say that a vector $d$ is normal to the hyperplane $H$ if $\langle d, x-y \rangle = 0$ for all $x,y \in H$.

\begin{lemma}\label{lem:nf}
Let  $\m{P}$ be full-dimensional and let $F$ be a face of $\m{P}$.
\begin{enumerate}
\item For any $x_0 \in F$,  $\langle a^F, x - x_0\rangle = 0$ for all $x \in F$ and $\langle a^F, x - x_0 \rangle < 0$ for all $x \not \in F$ if and only if $a^F \in \m{relint}(N_F)$.
%\item $\langle a^F, x - x_0\rangle < 0$, for all $x, x_0$ such that $x_0 \in F$ and  $x \not \in F$ and $\langle a^F, x - x_0\rangle \geq 0$ otherwise;
%\item For any $x_0 \not \in F$, $\langle a^F, x - x_0\rangle > 0$, for all $x \in F$ if and only if $a^F \in \m{relint}(N_F)$. and $\langle a^F, x - x_0\rangle \leq 0$ for all $x \not \in F$ if and only if $a^F \in \m{relint}(N_F)$. 
\item For any $x_0 \not \in F$, $\langle a^F, x - x_0\rangle > 0$ for all $x \in F$  and $\langle a^F, x - x_0\rangle \leq 0$ for all $x \not \in F$ if and only if $a^F \in \m{relint}(N_F)$. 
\item If $d \not \in \mathcal{N}(\m{P})$, then, for any $x_0 \in \m{P}$,
\[
\m{P}= S_{>,x_0} \uplus S_{=,x_0} \uplus S_{<,x_0}
\]
where $S_{>,x_0}$, $S_{=,x_0}$  and $S_{<,x_0}$ are disjoint, non-empty sets given by $\{ x \in \m{P} \colon \langle d, x - x_0 \rangle > 0 \}$, $ \{ x \in \m{P} \colon \langle d, x - x_0 \rangle  = 0 \}$ and $\{ x \in \m{P} \colon \langle d, x - x_0 \rangle > 0 \}$, respectively.
%For any $x_0 \in F$, if $a^F \in \m{rb}(N_F)$, then there exists a face $F'$ of $\m{P}$, of which $F$ is a proper face, such that $\langle a^F, x - x_0 \rangle = 0$ for all $x \in F'$.
\item $\m{rb}(N_F) = \biguplus_{F' \colon F' \supset F}\m{ri}(N_{F'})$, where the disjoint union ranges over all the faces $F'$ of $\m{P}$.
\item $N_F = \mathrm{cone} \left( a_1, \ldots, a_{m_F} \right)$, where $a_i$ denotes the transpose of the $i$-th row of the submatrix $A_F$ given in (\ref{eq:F}), $i =1\ldots,m_F$.
\end{enumerate}
\end{lemma}

\begin{proof} Recall that, since $\m{P}$ is full-dimensional, there is no vector $d \neq 0$ such that $\langle d, x - x_0 \rangle = 0$ for all pairs $x, x_0 \in \m{P}$. 
\begin{enumerate}
\item First we show sufficiency. If $a^F \in \m{relint}(N_F)$, then $a^F$ is a conic combination of all the rows of $A_F$ with positive coefficients. Therefore, $\langle a^F, x - x_0\rangle = 0$ for all $x \in F$, by the definition of $F$, and $\langle a^F, x - x_0\rangle < 0$ for all $x \not \in F$, since, in this case, $\langle a, x - x_0 \rangle < 0$ for some row $a$ of $A_F$. As for necessity,  if $a^F \in N_F$, then $\langle a^F, x - x_0\rangle < 0$ for all $x \in \m{ri}(\m{P})$. However, if $a^F \in \m{rb}(N_F)$, then $\langle a^F, x - x_0\rangle = 0$ for all $x \in F'$, where $F'$ is the face of $\m{P}$ such that $a^F \in \m{ri}(N_{F'})$. But then, since  $F \subset F'$, there exists a $x \not \in F$ for which $\langle a^F, x - x_0\rangle = 0$, which would produce a contradiction. Thus $a^F \not \in \m{rb}(N_F)$.

% and $\langle a^F, x - x_0\rangle \leq 0$ for all $x \in F'$, where $F'$ is a face of $\m{P}$ such that $F \subset F'$. In order for this inequality to hold for all points $x \in \m{rb}(\m{P}) \cap F^c$, it must be the case that $\langle a^F, x - x_0 \rangle = 0$ cannot hold if $x \in $

\item The proof is analogous to the previous case and is omitted.
\item  Since $d$ is not normal to any supporting hyperplane, the hyperplane $H = \{ x \colon \langle d, x - x_0 \rangle = 0 \}$ intersects $\m{P}$ is in its relative interior, and $\m{P}$ must have non-empty intersections with both the halfspaces $\{  x \in\mathbb{R}^k \colon \langle d, x - x_0 \rangle > 0 \}$ and $\{ x \in\mathbb{R}^k \colon \langle d, x - x_0 \rangle < 0 \}$ cut out by  $H$.
\item The claim follows directly from the definition of $N_F$ and the fact that $\mathcal{N}(\m{P})$ is a polyhedral complex \citep[see, e.g.,][]{STURM:95}, thus the relative boundary of $N_F$ is the disjoint union of the relative interiors of all its faces.
\item  Let $c \in  \mathrm{cone} \left( a_1, \ldots, a_{m_F} \right)$, so that $c = A_F^\top \lambda$, where $\lambda \in \mathbb{R}^k$ has nonnegative coordinates. Then, for all $x \in F$ and $y \in \m{P} \cap F^c$,
\[
\langle c, x \rangle = \langle \lambda, A_F x \rangle = \langle \lambda, b_F \rangle \geq \langle \lambda, A_F y \rangle = \langle A_F^\top \lambda, y \rangle = \langle c, y \rangle
\]
since $A_F x = b_F$ and $A_F y < b_F$. Thus, $c \in N_F$ and we have shown that $\mathrm{cone} \left( a_1, \ldots, a_{m_F} \right) \subseteq N_F$. Conversely, assume that $c$ is a nonzero vector in $N_F$ but $c \not \in \mathrm{cone} \left( a_1, \ldots, a_{m_F} \right)$. Then, $c$ is not normal to any supporting hyperplane of $F$, which implies that there exists a $x \in F$ and $y \in \m{P} \cap F^c$ such that $\langle c, x - y \rangle < 0$, producing a contradiction. Thus, it must be the case that $c \in \mathrm{cone} \left( a_1, \ldots, a_{m_F} \right)$ as well, yielding $N_F \subseteq \mathrm{cone} \left( a_1, \ldots, a_{m_F} \right)$.  
\end{enumerate}
\end{proof}

\section{Appendix C: Checking for the existence of the MLE via Linear Programming.}
Deciding whether the MLE exists, that is, whether the vector of observed sufficient statistics $x$ is such that $x \in \m{ri}(\m{P})$  is particularly simple if one has  access to a  $\mathcal{H}$ representation of $\m{P}$ as in (\ref{eq:P}), as indicated in the next result, of immediate verification.
\begin{lemma}
The MLE exists  if and only if the system $ A x \leq b$ is satisfied with strict inequalities. 
\end{lemma}
Unfortunately, this type of representation is typically not available or prohibitively hard to compute, even when $k$ is small, since $\m{P}$ may have a number of faces that grow super-exponentially in $k$ \citep[see, for example,][]{ZIE:01}.

If instead only a $\mathcal{V}$ representation (\ref{eq:Vrepre}) is available or computable, the existence of the MLE can be established using linear programming, as outlined below.  Let $B$ be a matrix whose columns contain the vertices and extreme rays of $\m{P}$, namely the vectors in $\mathcal{Q}$ and $\mathcal{C}$ from Equation (\ref{eq:Vrepre}). Then $x \in \m{ri}(\m{P})$ if and only if $x$ can be obtained as a linear combinations of the vectors in  $\mathcal{Q}$ and $\mathcal{C}$ with strictly positive coefficients.
\begin{lemma}
The MLE exists if and only if $x = B z$, for a vector $z$ with strictly positive coordinates. 
\end{lemma}

This is a feasibility problem which can be decided by solving the linear program
\[
\begin{array}{crcl}
 &  \max s & & \\
\mathrm{s.t.} & B z & = & x\\
& z_i - s & \geq & 0\\
& s & \geq & 0,
\end{array}
\]
where $z_i$ denotes the $i$-th coordinate of $z$ and $s$ is a scalar. If $(s^*,z^*)$ is the optimum, then the MLE exists  if and only if  $s^* >0$.

An alternative  linear program, which may be computationally preferable, can be formulated based on Theorem \ref{thm:gordan}, whose proof can be found in \cite{SCH:98}, as follows:
\[
\begin{array}{crcl}
 &  \max \langle 1, y \rangle & & \\
\mathrm{s.t.} & B^\top y & = & 0\\
& y & \geq &0\\
& y & \leq & 1. \\
\end{array}
\]
If $y^*$ is the optimum, the MLE does not exist if and only if $\langle 1, y^* \rangle > 0$.

\begin{theorem}[{\bf Gordan's Theorem of Alternatives}]\label{thm:gordan}
Given a matrix $B$, the following are alternatives:
\begin{enumerate}
\item $B x > 0$ has a solution $x$.
\item $B^{\top} y = 0$, $y \gneq 0$, has a solution $y$.
\end{enumerate}
\end{theorem}

\section{Appendix D: Software}\label{sec:soft}
The code used for the analysis and for the figures of the paper is available on the web at
\begin{center}
 \url{http://www.stat.cmu.edu/~arinaldo/ERG/}
\end{center}

The software includes:
\begin{enumerate}
\item the {\tt MATLAB} GUI used for creating Figures \ref{fig:gui1}, \ref{fig:gui2} and \ref{fig:gui3} and some short movies showing the relationship between sequences of natural parameters moving along the outer normals of $\m{P}$ and the corresponding sequences of mean values;
\item an MPI {\tt C++} program for complete enumeration of all undirected graphs on $n$ nodes and for counting the number of edges, triangles, $k$-stars and alternating $k$-stars. However, complete enumeration is only feasible only for very small graph. Using our program, which can certainly be be improved, it took about 1 hour on a 64-node cluster to enumerate all graphs on $9$ nodes, but for the $10$-node graph, the estimated running time is about 26.5 days.
%\item enumerations of all graphs on $5$, $6$, $7$, $8$ and $9$ nodes, along with the number of edges, triangles and 2-stars.
\end{enumerate}

\end{document}